\documentclass[letterpaper, 10 pt, conference]{cls/ieeeconf}  
\IEEEoverridecommandlockouts
\let\labelindent\relax

\usepackage{enumitem}
\usepackage{amsmath,amssymb,amsfonts,amsthm,dsfont} 
\usepackage{algorithm,algorithmicx,listings}        
\usepackage[noend]{algpseudocode}			        
\usepackage{graphicx,tabularx,adjustbox}
\usepackage[font={small}]{caption}   
\usepackage{subcaption}
\usepackage{xcolor}

\newtheorem{proposition}{Proposition}
\newtheorem*{proposition*}{Proposition}
\newtheorem{corollary}{Corollary}
\newtheorem*{corollary*}{Corollary}

\theoremstyle{definition}

\newtheorem*{assumption*}{Assumption}
\newtheorem*{problem*}{Problem}
\newtheorem{problem}{Problem}
\theoremstyle{remark}

\newtheorem*{solution*}{Solution}

\newcommand{\prl}[1]{\left(#1\right)}

\newcommand{\scaleMathLine}[2][1]{\resizebox{#1\linewidth}{!}{$\displaystyle{#2}$}}

\newcommand{\NEW}[1]{{\color{black}#1}}

\def\thetitle{Autonomous Navigation in Unknown Environments using Sparse Kernel-based Occupancy Mapping}

\newcommand{\calA}{{\cal A}}
\newcommand{\calB}{{\cal B}}

\newcommand{\calD}{{\cal D}}

\newcommand{\calV}{{\cal V}}

\newcommand{\bfa}{\mathbf{a}}

\newcommand{\bfp}{\mathbf{p}}

\newcommand{\bfs}{\mathbf{s}}

\newcommand{\bfv}{\mathbf{v}}

\newcommand{\bfx}{\mathbf{x}}
\newcommand{\bfy}{\mathbf{y}}
\newcommand{\bfz}{\mathbf{z}}

\title{\LARGE \bf \thetitle%
\thanks{We gratefully acknowledge support from ARL DCIST CRA W911NF-17-2-0181 and ONR N00014-18-1-2828.}%
}

\author{Thai Duong \and Nikhil Das \and Michael Yip \and Nikolay Atanasov
\thanks{The authors are with the Department of Electrical and Computer Engineering, University of California, San Diego, La Jolla, CA 92093 USA {\tt\small \{tduong, nrdas, yip, natanasov\}@ucsd.edu}}%
}

\begin{document}
\maketitle
\thispagestyle{empty}
\pagestyle{empty}

\begin{abstract}
This paper focuses on real-time occupancy mapping and collision checking onboard an autonomous robot navigating in an unknown environment. We propose a new map representation, in which occupied and free space are separated by the decision boundary of a kernel perceptron classifier. We develop an online training algorithm that maintains a very sparse set of support vectors to represent obstacle boundaries in configuration space. We also derive conditions that allow complete (without sampling) collision-checking for piecewise-linear and piecewise-polynomial robot trajectories. We demonstrate the effectiveness of our mapping and collision checking algorithms for autonomous navigation of an Ackermann-drive robot in unknown environments.
\end{abstract}


\section{Introduction}
\label{sec:intro}


Autonomous navigation in robotics involves localization, mapping, motion planning, and control in a partially known environment perceived through streaming data from onboard sensors~\cite{human_friendly_nav_guzzi_icra13, safe_auto_nav_pavone_rss18}. In this paper, we focus on the mapping problem and, specifically, on enabling large-scale, yet compact, representations and efficient collision checking to support autonomous navigation. Existing work uses a variety of map representations based on voxels~\cite{thrun2005probabilistic,octomap,ieee_map_rep_standard,voxblox}, surfels~\cite{behley2018rss_surfel_slam}, geometric primitives~\cite{Kaess_infiniteplanes}, objects~\cite{Bowman_SemanticSLAM_ICRA17}, etc. 

We propose a novel mapping method that uses a kernel perceptron model to represent the occupied and free space of the environment. The model uses a set of support vectors to represent obstacle boundaries in configuration space. The complexity of this representation scales with the complexity of the obstacle boundaries rather than the environment size. We develop an online training algorithm to update the support vectors incrementally as new depth observations of the local surroundings are provided by the robot's sensors. To enable motion planning in the new occupancy representation, we develop an efficient collision checking algorithm for piecewise-linear and piecewise-polynomial trajectories in configuration space.

\textbf{Related Work.} Occupancy grid mapping is a commonly used approach for modeling the free and occupied space of an environment. The space is discretized into a collection of cells, whose occupancy probabilities are estimated online using the robot's sensory data. While early work~\cite{thrun2005probabilistic} assumes that the cells are independent, Gaussian process (GP) occupancy mapping~\cite{OCallaghan2012gaussian, wang2016fast, jadidi2017warped} uses a kernel function to capture the correlation among grid cells and predict the occupancy of unobserved cells. Online training of a Gaussian process model, however, does not scale well as its computational complexity grows cubically with the number of data points. Ramos et al.~\cite{ramos2016hilbertmap} improve on this by projecting the data points into Hilbert space and training a logistic regression model. Lopez and How~\cite{lopez2017aggressive} propose an efficient determinstic alternative, which builds a k-d tree from point clouds and queries the nearest obstacles for collision checking. Using spatial partitioning similar to a k-d tree, octree-based maps~\cite{octomap,chen2017improving} offer efficient map storage by performing octree compression. Meanwhile, AtomMap~\cite{fridovich2017atommap} stores a collection of spheres in a k-d tree as a way to avoid grid cell discretization of the map.

Navigation in an unknown environment, requires the safety of potential robot trajectory to be evaluated through a huge amount of collision checks with respect to the map representation~\cite{bialkowski2016efficient,luo2014empirical,hauser2015lazy}. Many works rely on sampling-based collision checking, simplifying the safety verification of continuous-time trajectories by evaluating only a finite set of samples along the trajectory~\cite{Tsardoulias2016planninggridmap, luo2014empirical}. This may be undesirable in safety critical applications. Bialkowski et al.~\cite{bialkowski2016efficient} propose an efficient collision checking method using safety certificates with respect to the nearest obstacles. Using a different perspective, learning-based collision checking methods~\cite{das2017fastron, pan2015efficient, huh2016learningGMM} sample data from the environment and train machine learning models to approximate the obstacle boundaries. Pan et al.~\cite{pan2015efficient} propose an incremental support vector machine model for pairs of obstacles but train the models offline. Closely related to our work, Das et al.~\cite{das2017fastron, das2019learning} develop an online training algorithm, called Fastron, to train a kernel perceptron collision classifier. To handle dynamic environments, Fastron actively resamples the environment and updates the model globally. Geometry-based collision checking methods, such as the Flexible Collision Library (FCL)~\cite{pan2012fcl}, are also related but rely on mesh representations of the environment which may be inefficient to generate from local observations.

Inspired by GP mapping techniques, we utilize a radial basis function (RBF) kernel to capture occupancy correlations but focus on a compact representation of obstacle boundaries using kernel perceptron. Furthermore, motivated by the safety certificates in~\cite{bialkowski2016efficient}, we derive our own safety guarantees for efficient collision checking algorithms.

\textbf{Contributions}. This paper introduces a sparse kernel-based mapping method that: 
\begin{itemize}[leftmargin=2em,nosep]
  \item represents continuous-space occupancy using a sparse set of support vectors stored in an $R^*$-tree data structure, scaling efficiently with the complexity of obstacle boundaries (Sec. \ref{subsec:ogm_with_fastron}),

  \item allows online map updates from streaming partial observations using our proposed incremental kernel perceptron training algorithm built on the Fastron model (Sec. \ref{subsec:ogm_with_fastron}), and

  \item provides efficient and complete (without sampling) collision checking for piecewise-linear and piecewise-polynomial trajectories with safety guarantees based on  nearest support vectors queried from the $R^*$-tree (Sec. \ref{subsec:collison_checking_with_fastron_map} and \ref{sec:auto_nav}).

\end{itemize}

\section{Problem Formulation}
\label{sec:problem_formulation}
Consider a spherical robot with center $\bfs \in \mathcal{S}:=[0,1]^d$ and radius $r \in \mathbb{R}_{>0}$ navigating in an unknown environment. Let $\mathcal{S}_{obs}$ and $\mathcal{S}_{free}$ be the obstacle space and free space in $\mathcal{S}$, respectively. In configuration space (C-space) $\mathcal{C}$, the robot body becomes a point $\bfs$, while the obstacle space and free space are transformed as $\mathcal{C}_{obs} = \cup_{\bfx\in \mathcal{S}_{obs}} \mathcal{B}(\bfx,r)$, \NEW{where $\mathcal{B}(\bfx,r) = \{x'\in \mathcal{S}: \|x-x'\|_2 \leq r\}$,} and $\mathcal{C}_{free} = \mathcal{S} \setminus \mathcal{C}_{obs}$.
 Assume that the robot position $\bfs_{t_k} \in \mathcal{C}$ at time $t_k$ is known or provided by a localization algorithm. Let $\bfs_{t_{k+1}} = f(\bfs_{t_k}, \bfa_{k})$ characterize the robot dynamics for an action $\bfa_{k} \in \mathcal{A}$. Applying $\bfa_{k}$ at $\bfs_{t_k}$ also incurs a motion cost $c(\bfs_{t_k}, \bfa_{k})$ (e.g., distance
or energy).
 The robot is equipped with a depth sensor that provides distance measurements $\bfz_{t_k}$ to the obstacle space $\mathcal{S}_{obs}$ within its field of view. Our objective is to construct an occupancy map $\hat{m}_{t_k}: \mathcal{C} \rightarrow \{-1, 1\}$ of the C-space based on accumulated observations $\bfz_{t_{0:k}}$, where ``-1" and ``1" mean ``free" and ``occupied", respectively. Assuming unobserved regions are free, we rely on $\hat{m}_{t_k}$ to plan a robot trajectory to a goal region $\mathcal{C}_{goal} \subset \mathcal{C}_{free}$. As the robot is navigating, new sensor data is used to update the map and recompute the motion plan. In this online setting, the map update, $\hat{m}_{t_{k+1}} = g(\hat{m}_{t_k}, \bfz_{t_k})$, is a function of the previous estimate $\hat{m}_{t_k}$ and a newly received depth observation~$\bfz_{t_k}$.

\begin{problem}
\label{problem_formulation_unknown_env}
Given a start state $\bfs_0 \in \mathcal{C}_{free}$ and a goal region $\mathcal{C}_{goal} \subset \mathcal{C}_{free}$, find a sequence of actions that leads the robot to $\mathcal{C}_{goal}$ safely, while minimizing the motion cost:
\begin{align}
\label{problem_formulation_unknown_env_equation}
\min_{N, \bfa_0, \ldots, \bfa_N} \;&\sum_{k=0}^{N-1}  c(\bfs_{t_k}, \bfa_{k})\\
\text{s.t.} \quad\; &\bfs_{t_{k+1}} = f(\bfs_{t_k}, \bfa_{k}), \hat{m}_{t_{k+1}} = g(\hat{m}_{t_k}, \bfz_{t_k}), \notag\\
&\;\bfs_{t_N} \in \mathcal{C}_{goal}, \;\hat{m}_{t_k}(\bfs_{t_k}) = - 1,\;k= 0,\ldots,N.\notag
\end{align}
\end{problem}

\section{Preliminaries}
\label{sec:prelim}
\NEW{In this section, we provide a summary on kernel perceptron and Fastron which is useful for our derivations in the next sections.} 
The \emph{kernel perceptron} model is used to classify a set of $N$ labeled data points. For $l = 1, \ldots, N$, a data point $\bfx_l$ with label $\bfy_l \in \{-1, 1\}$ is assigned a weight $\alpha_l \in \mathbb{R}$. Training determines a set of $M^+$ positive support vectors and their weights $\Lambda^+ = \{(\bfx_i, \alpha_i)\}$ and a set of $M^-$ negative support vectors and their weights $\Lambda^- = \{(\bfx_j^-, \alpha_j^-)\}$. The decision boundary is represented by a score function, 
\begin{equation}
\label{eq:fastron_score}
F(\bfx) =\sum_{i = 1}^{M^+} \alpha_i^+ k(\bfx_i^+, \bfx) - \sum_{j = 1}^{M^-} \alpha_j^- k(\bfx_j^-, \bfx),
\end{equation}
where $k(\cdot, \cdot)$ is a kernel function and $\alpha_j^-, \alpha_i^+ > 0$. The sign of $F(\bfx)$ is used to predict the class of a test point $\bfx$. 

\emph{Fastron}~\cite{das2017fastron, das2019learning} is an efficient training algorithm for the kernel perceptron model. It prioritizes updating misclassified points based on their margins instead of random selection as in the original kernel perceptron training. Our previous work~\cite{das2017fastron, das2019learning} shows that if $\alpha_l = \xi y_l - (\sum_{i\neq l} \alpha^+_i k(\bfx^+_i, \bfx_l) - \sum_{j\neq l} \alpha^-_j k(\bfx^-_j, \bfx_l))$ for some $\xi > 0$, then $\bfx_l$ is correctly classified with label $y_l$. Based on this fact, Fastron utilizes one-step weight correction ${\Delta \alpha = \xi y_l - (\sum \alpha^+_i k(\bfx^+_i, \bfx_l) - \sum \alpha^-_j k(\bfx^-_j, \bfx_l))}$ where $\xi = \xi^+$ if $y_l = 1$ and $\xi = \xi^-$ if $y_l = -1$.

\section{Sparse Kernel-based Occupancy Mapping}
\label{subsec:ogm_with_fastron}
 \addtolength{\textfloatsep}{-7mm}
\begin{figure*}[h]
\centering
\begin{subfigure}[t]{0.25\textwidth}
        \centering
        \includegraphics[width=\textwidth]{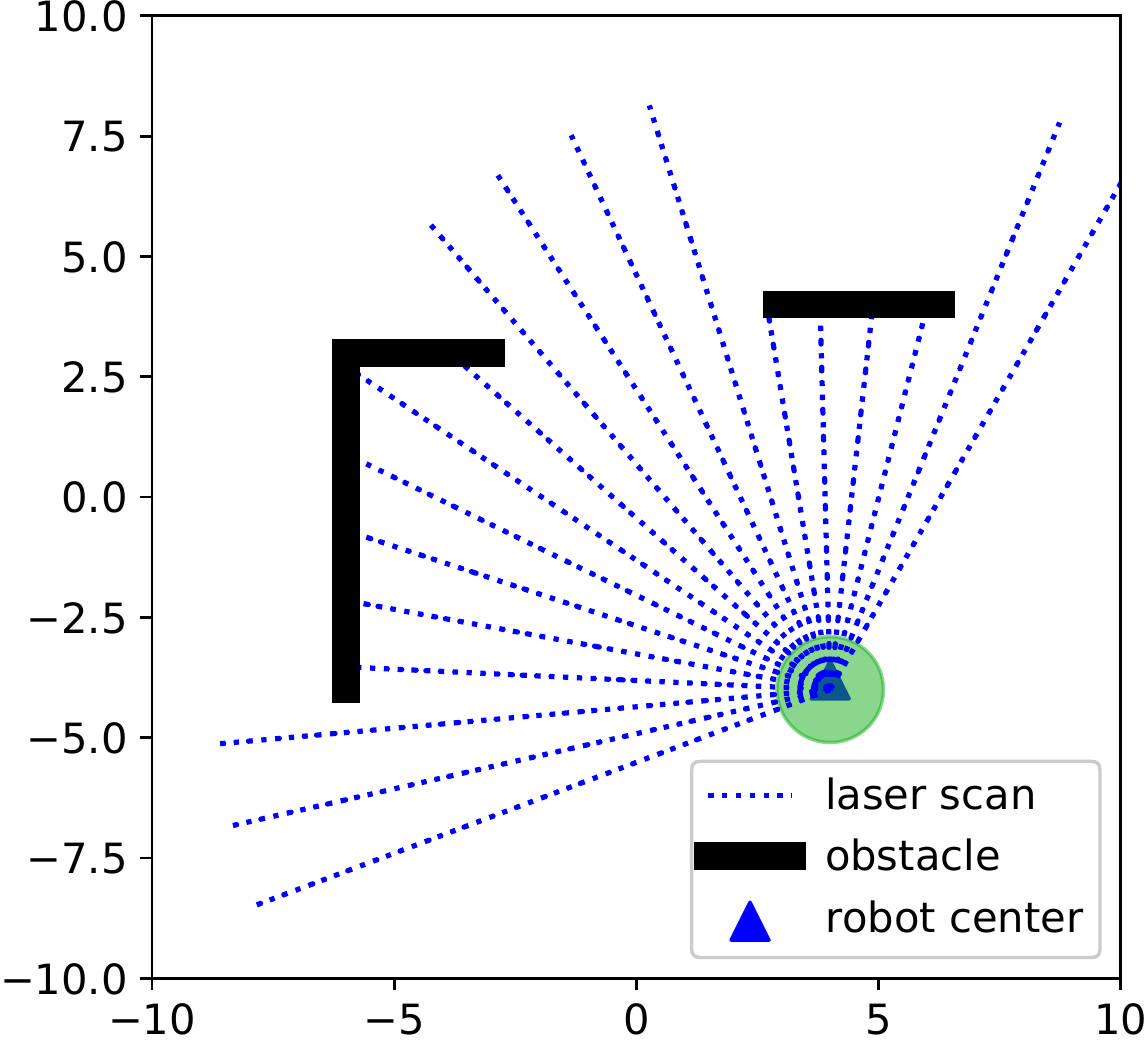}
        \vspace{-6mm}
        \caption{}
        \label{fig:laser_scan_ws}
\end{subfigure}%
\begin{subfigure}[t]{0.25\textwidth}
        \centering
        \includegraphics[width=\textwidth]{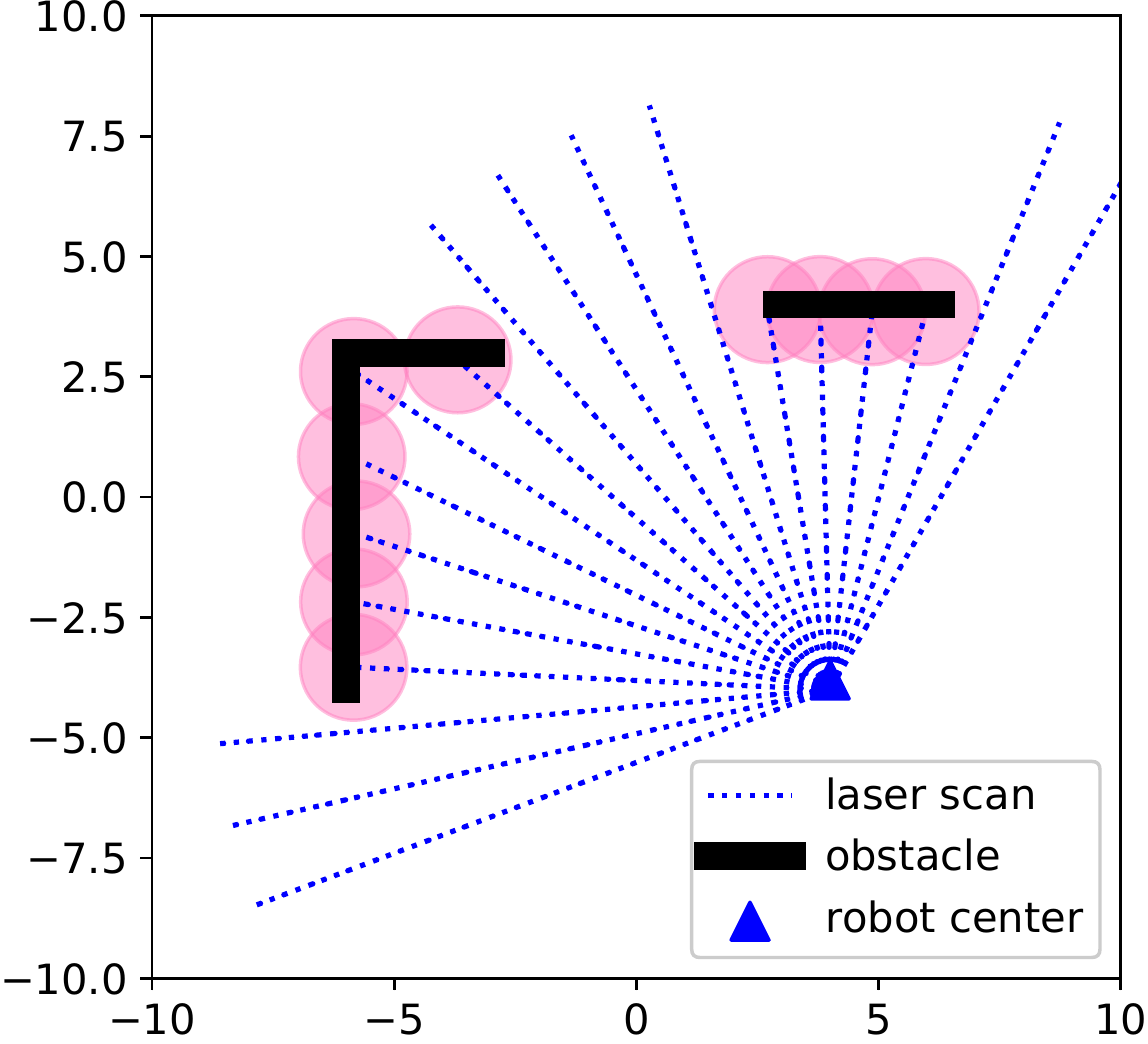}
        \vspace{-6mm}
        \caption{}
        \label{fig:laser_scan_cs}
\end{subfigure}%
\begin{subfigure}[t]{0.25\textwidth}
        \centering
        \includegraphics[width=\textwidth]{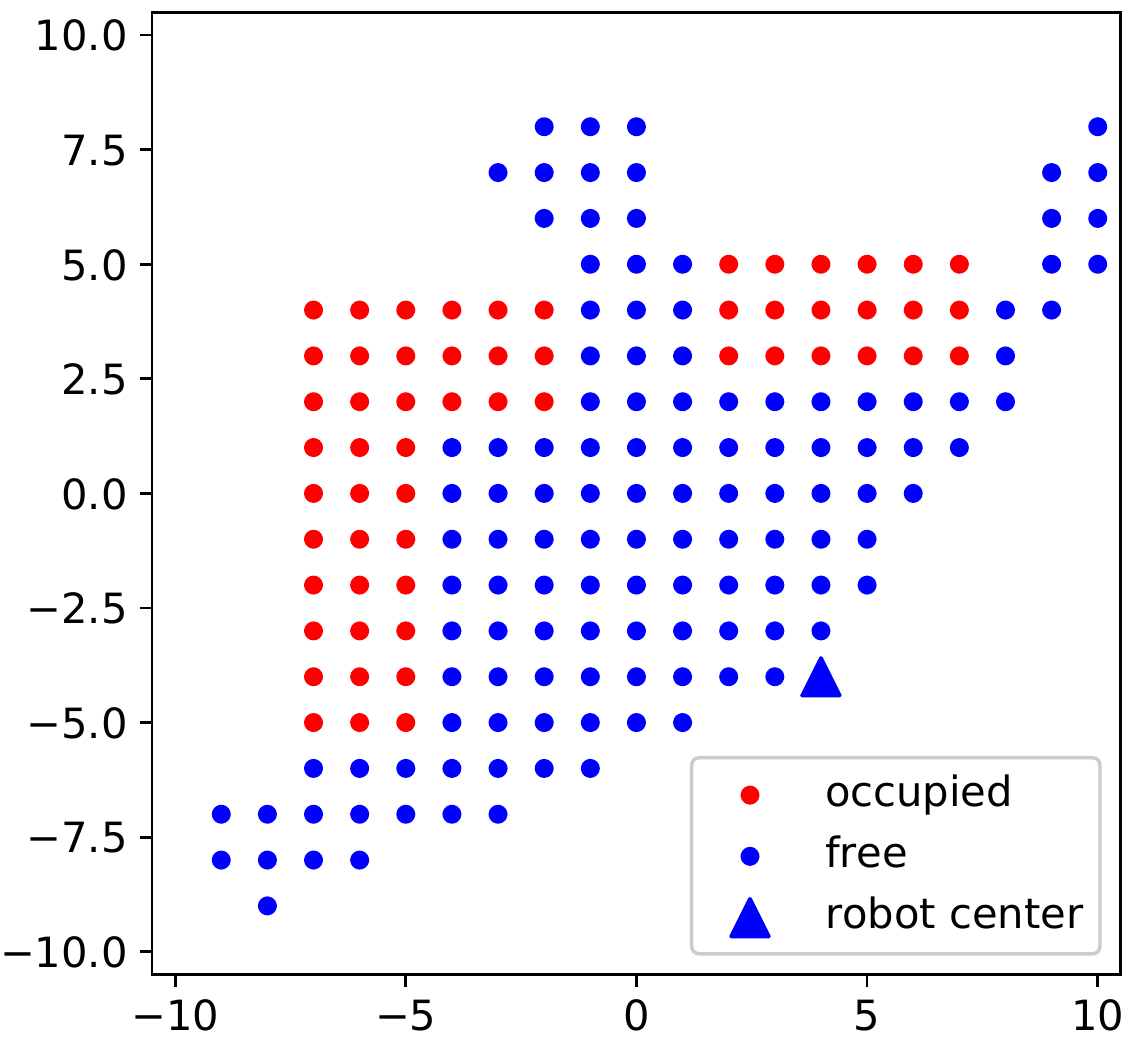}
        \vspace{-6mm}
        \caption{}
        \label{fig:orig_data}
\end{subfigure}%
\begin{subfigure}[t]{0.25\textwidth}
        \centering
        \includegraphics[width=\textwidth]{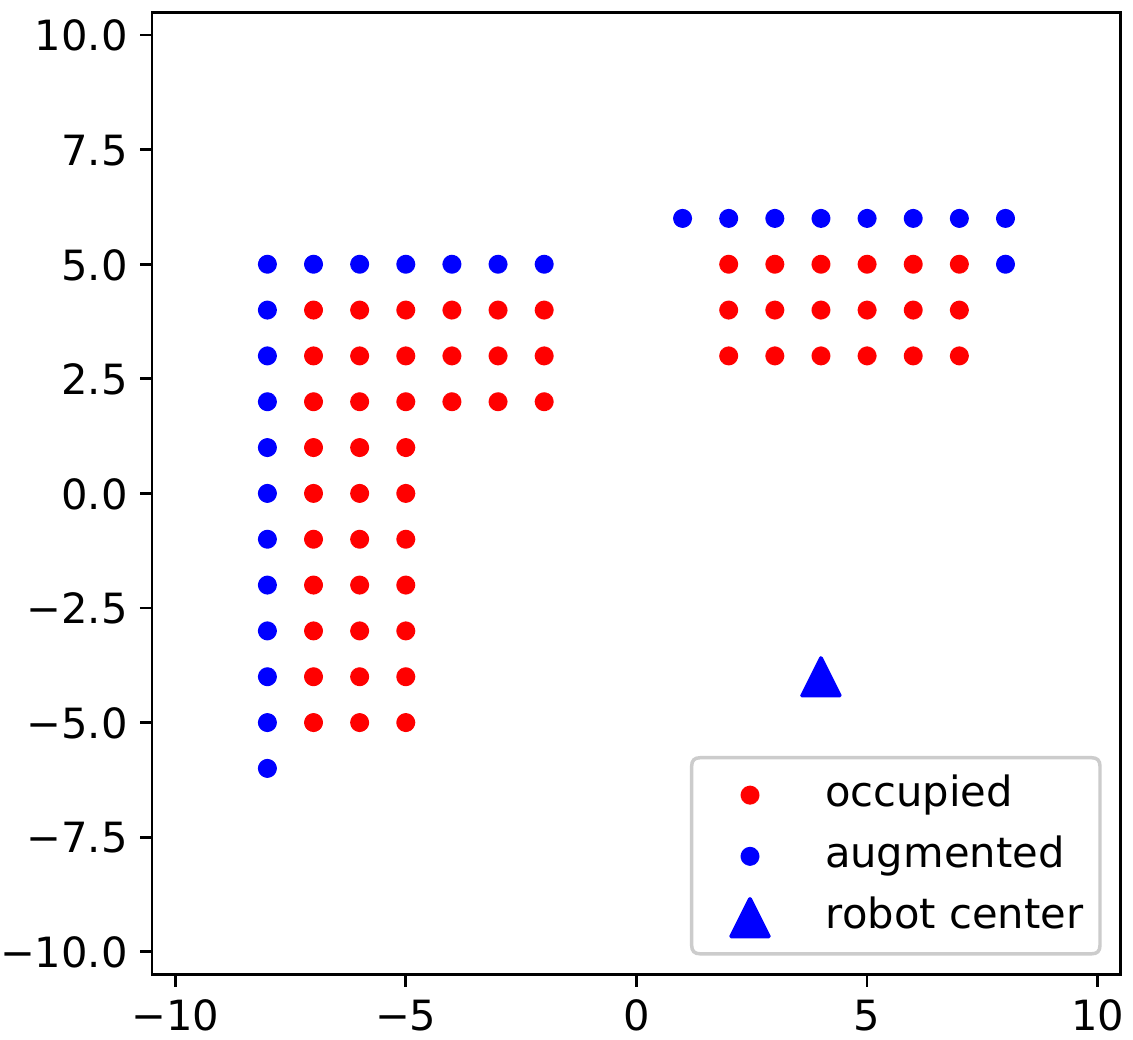}
        \vspace{-6mm}
        \caption{}
        \label{fig:augmented_data}
\end{subfigure}%

\begin{subfigure}[t]{0.25\textwidth}
        \centering
        \includegraphics[width=\textwidth]{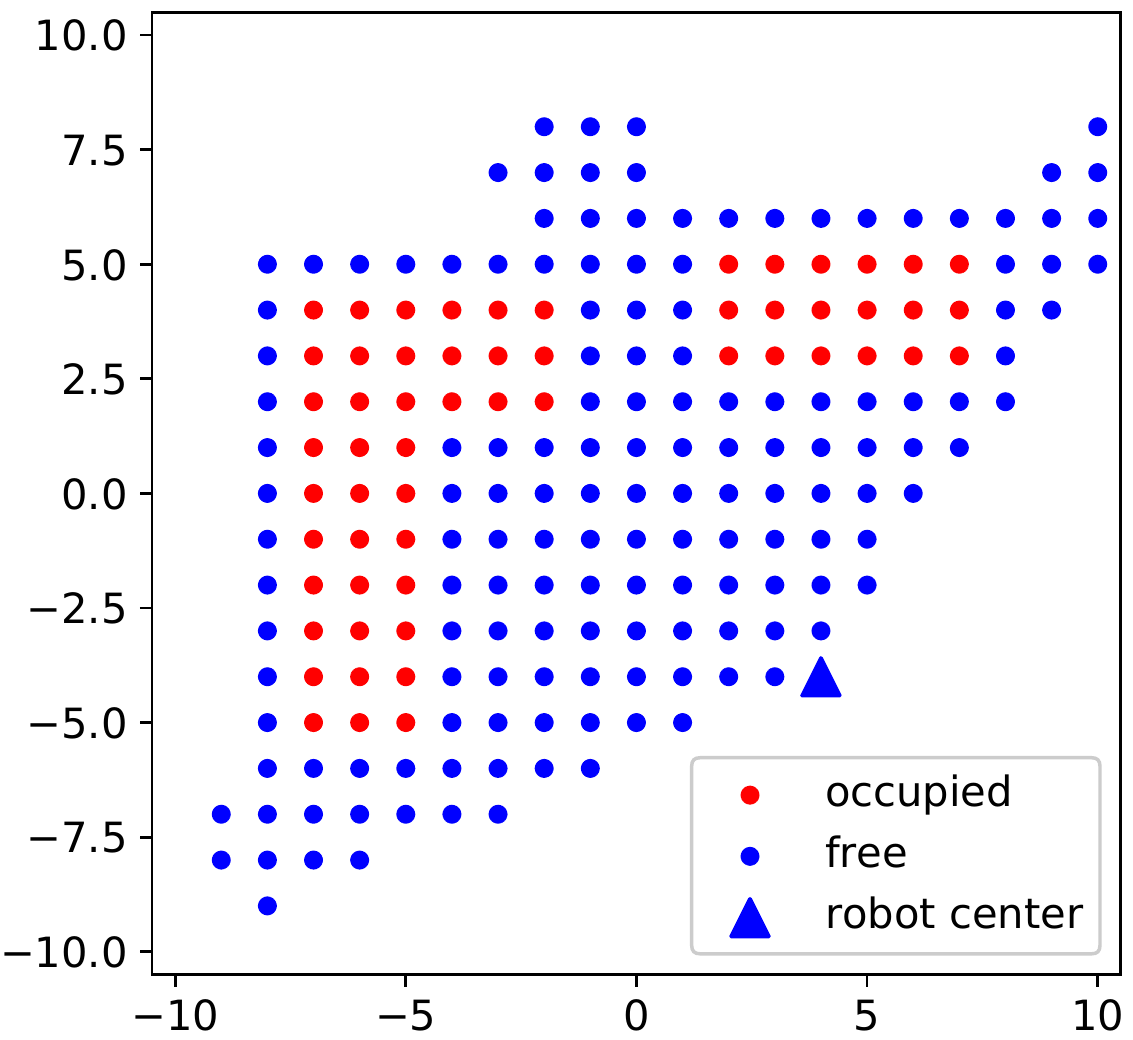}
        \vspace{-6mm}
        \caption{}
        \label{fig:training_data}
\end{subfigure}%
\begin{subfigure}[t]{0.25\textwidth}
        \centering
\includegraphics[width=\textwidth]{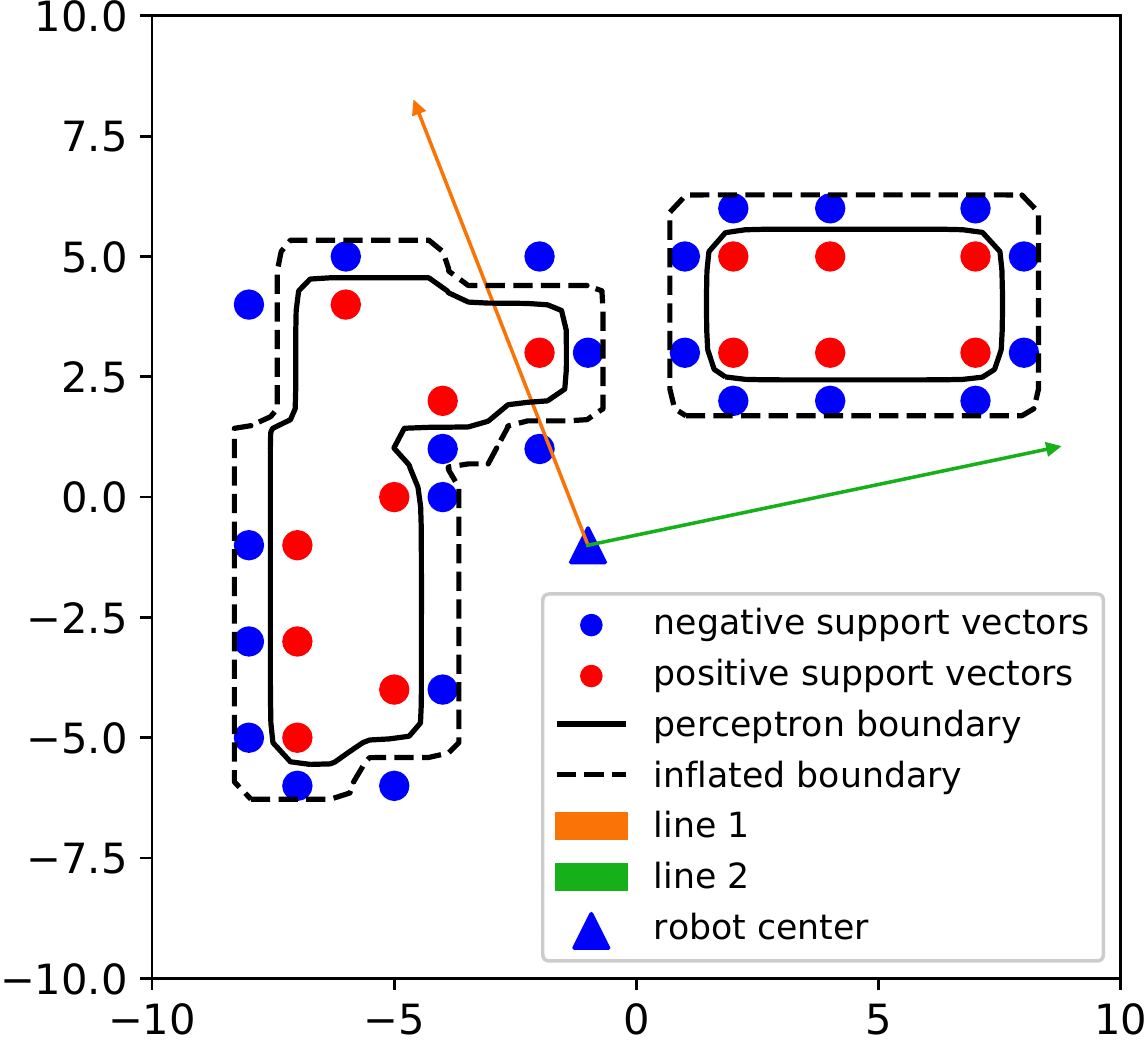}%
\vspace{-2mm}
        \caption{}
        \label{fig:boundary_example}
\end{subfigure}%
\begin{subfigure}[t]{0.25\textwidth}
        \centering
\includegraphics[width=0.97\textwidth]{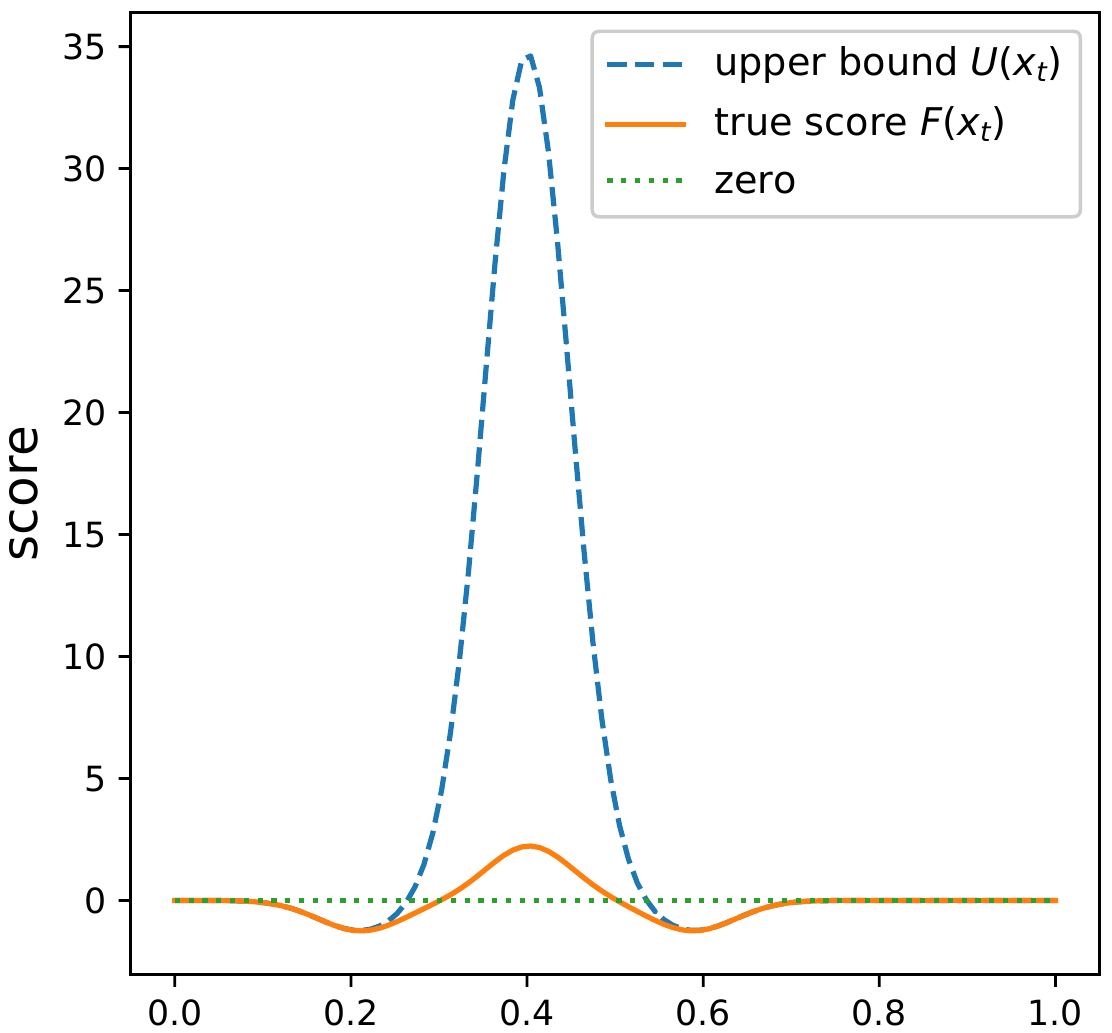}%
\vspace{-2mm}
        \caption{}
        \label{fig:upperbound_line1}
\end{subfigure}%
\begin{subfigure}[t]{0.25\textwidth}
        \centering
\includegraphics[width=1.01\textwidth]{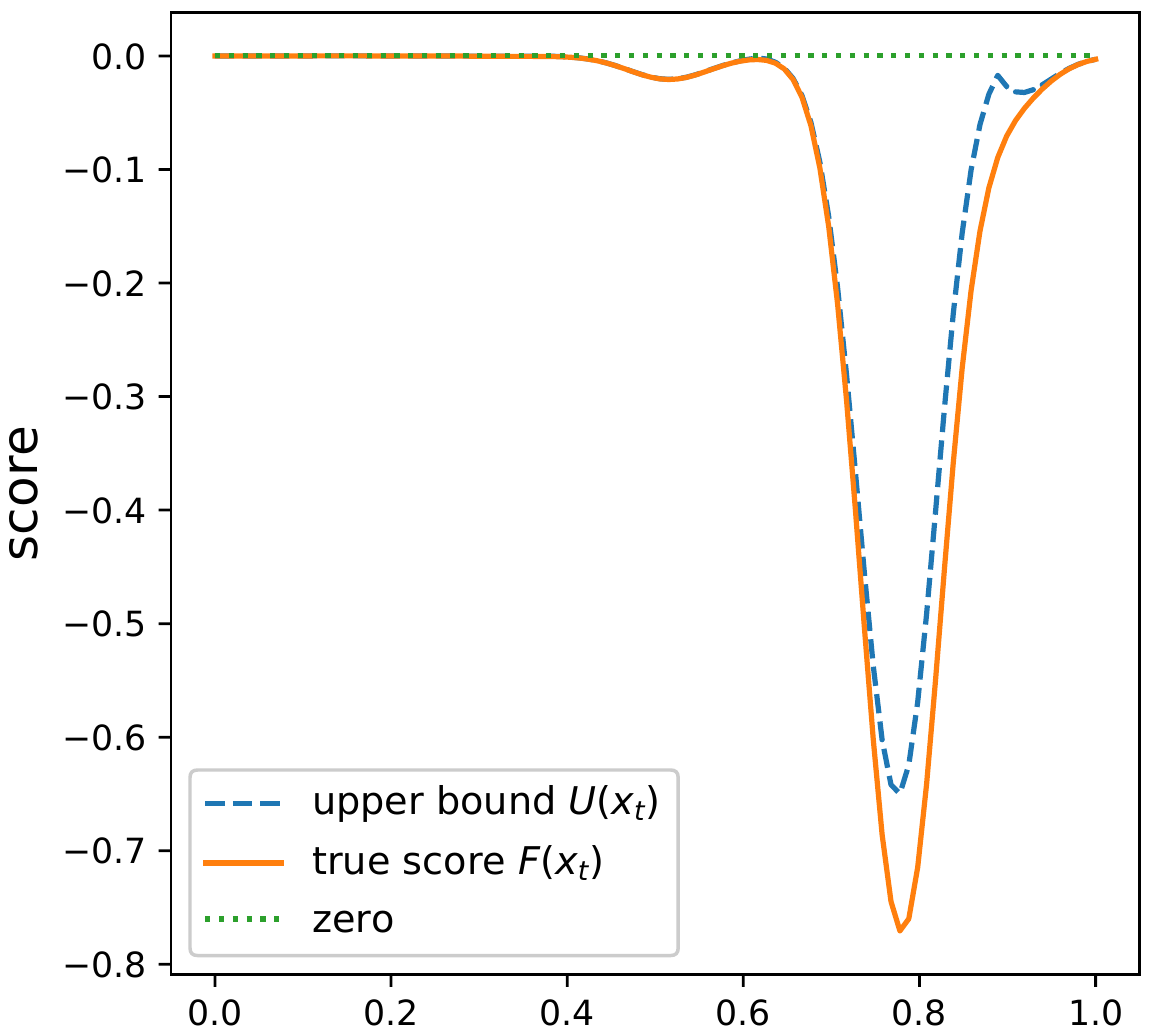}
\vspace{-6mm}
        \caption{}
        \label{fig:upperbound_line2}
\end{subfigure}%
\caption{Example of our mapping method: (a) scan in work space; (b) scan in C-space; (c) samples from scan; (d) augmented free points (e) training data $\mathcal{D}$ from one lidar scan; (f) the exact decision boundary generated by the score $F(\bfx)$ and the inflated boundary generated by the upper bound $U(\bfx)$; $F(\bfx)$ and $U(\bfx)$ along two rays: (g) one that enters the occupied space and (h) one that remains obstacle-free.}
\vspace{-6mm}
\label{fig:mapping_method}
\end{figure*}
We propose a new version of the Fastron algorithm, utilizing only local streaming data, to achieve real-time sparse kernel-based occupancy mapping. For online learning, Fastron resamples a \textit{global} training dataset around the current support vectors and updates the support vectors. In our setting, only \textit{local} data in the robot's vicinity is available from the onboard sensors. We propose an incremental version of Fastron in Alg.~\ref{alg:fastron_model} such that: 1) training is performed with \textit{local} data $\mathcal{D}_t$ generated from a depth measurement $\bfz_t$ at a time $t$ and 2) the support vectors are stored in an $R^*$-tree data structure, enabling efficient score function~\eqref{eq:fastron_score} computations.


\begin{algorithm}[t!]
\caption{Incremental Fastron Training with Local Data}
\label{alg:fastron_model}
  \footnotesize
	\begin{algorithmic}[1]
		\Require Sets $\Lambda^+ = \{(\bfx^+_i, \alpha^+_i)\}$ and $\Lambda^- = \{(\bfx_j^-, \alpha_j^-)\}$ of $M^+$ positive and $M^-$ negative support vectors stored in an $R^*$-tree; Local dataset $\mathcal{D} = \{(\bfp_l, q_l)\}$ generated from location $\bfs_t$; $\xi^+, \xi^- > 0$; $N_{max}$
		\Ensure Updated $\Lambda^+, \Lambda^-$.
		\State \NEW{Get $K^+, K^-$ nearest negative and positive support vectors.}
		\For {$(\bfp_l, q_l)$ in $\mathcal{D}$}
			\State Calculate $F_l = \sum_{i = 1}^{K^+} \alpha^+_i k(\bfx^+_i,\bfp_l) - \sum_{j = 1}^{K^-} \alpha^-_j k(\bfx^-_j,\bfp_l)$
		\EndFor
		\For {$t = 1$ to $N_{max}$}
			\If{$q_lF_l > 0 \quad \forall l$} \Return{$\Lambda^+, \Lambda^-$} \Comment{Margin-based priotization}
			\EndIf
			\State $m = \text{argmin}_l q_l F_l$	
			\If{$q_m > 0$}  $\Delta\alpha = \xi^+q_m -F_m$ \Comment{One-step weight correction}
			\Else $\quad\Delta\alpha = \xi^-q_m -F_m$
			\EndIf
			\If{$\exists(\bfp_m, \alpha^+_m) \in \Lambda^+$}
			$\alpha^+_m \text{+=} \Delta \alpha$, $F_l \text{+=} k(\bfp_l, \bfp_m) \Delta\alpha, \forall l$
			\ElsIf{$\exists (\bfp_m, \alpha^-_m) \in \Lambda^-$}
			$\alpha^-_m \text{-=} \Delta \alpha$, $F_l \text{-=} k(\bfp_l, \bfp_m) \Delta\alpha, \forall l$
				\ElsIf{$q_m > 0$} $\alpha^+_m = \Delta \alpha$, $\Lambda^+ = \Lambda^+ \cup \{(\bfp_m, \alpha^+_m)\}$
					\Else{$\quad\alpha^-_m = -\Delta \alpha$, $\Lambda^- = \Lambda^- \cup \{(\bfp_m, \alpha^-_m)\}$}
			\EndIf
			\For{$(\bfp_l, q_l) \in \mathcal{D}$} \Comment{Remove redundant support vectors}
				\If{$\exists (\bfp_l, \alpha^+_l) \in \Lambda^+$ and $q_l(F_l - \alpha^+_l) > 0$}
				\State $\Lambda^+ \!= \Lambda^+ \setminus \{(\bfp_l, \alpha^+_l)\}, F_n \text{-=}\; k(\bfp_l, \bfp_n)\alpha^+_l, \forall (\bfp_n, \cdot) \in \mathcal{D}$
				\EndIf
				\If{$\exists (\bfp_l, \alpha^-_l) \in \Lambda^-$ and $q_l(F_l + \alpha^-_l) > 0$}
				\State $\Lambda^- \!= \Lambda^- \setminus \{(\bfp_l, \alpha^-_l)\}, F_n \text{+=}\; k(\bfp_l, \bfp_n)\alpha^-_l\!, \forall (\bfp_n, \cdot) \!\in \mathcal{D}$
				\EndIf
			\EndFor
		\EndFor
		\State \Return{$\Lambda^+, \Lambda^-$}
	\end{algorithmic}
\end{algorithm}
\textbf{Data Generation}. Fig.~\ref{fig:laser_scan_ws} illustrates a lidar scan $\bfz_t$ obtained by the robot at time $t$. In configuration space, each laser ray end point corresponds to a ball-shaped obstacle, while the robot body becomes a point as shown in Fig.~\ref{fig:laser_scan_cs}. To generate local training data $\mathcal{D}_t$, the occupied and free C-space areas observed by the lidar are sampled (e.g., on a regular grid). As shown in Fig.~\ref{fig:orig_data}, this generates a set $\bar{\mathcal{D}}_t$ of points with label ``1'' (occupied) in the ball-shaped occupied areas and with label ``-1'' (free), outside. As unobserved areas are assumed free, neighboring points to the occupied samples in $\bar{\mathcal{D}}_t$ that are not already in $\bar{\mathcal{D}}_t$ or in the support vectors are added to an augmented set $\tilde{\mathcal{D}}_t$ with label ``-1''.    The augmented dataset $\tilde{\mathcal{D}}_t$ is illustrated in Fig.~\ref{fig:augmented_data} assuming the set of support vectors is empty. The local data $\mathcal{D}_t = \bar{\mathcal{D}}_t \cup \tilde{\mathcal{D}}_t$ (Fig.~\ref{fig:training_data}) is used in our Incremental Fastron Algorithm to update the support vectors (Fig.~\ref{fig:boundary_example}). Storing the sets of support vectors $\Lambda^+$, $\Lambda^-$ over time requires significantly less memory than storing the training data $\cup_t \mathcal{D}_t$. The occupancy of a query point $\bfx$ can be estimated from the support vectors by evaluating the score function $F(\bfx)$ in Eq.~\eqref{eq:fastron_score}. Specifically, $\hat{m}_t(\bfx) = -1$ if $F(\bfx) < 0$ and $\hat{m}_t(\bfx) = 1$ if $F(\bfx) \geq 0$. Fig.~\ref{fig:boundary_example} illustrates the boundaries generated by Alg.~\ref{alg:fastron_model}.

\textbf{Score Approximation}. As the robot explores the environment, the number of support vectors required to represent the obstacle boundaries increases. Since the score function~\eqref{eq:fastron_score} depends on all support vectors, the time to train the kernel perceptron model online would increase as well. We propose an approximation to the score function $F(\bfx)$ under the assumption that the kernel function $k(\bfx_i, \bfx_j)$ is \emph{radial} (depends only on $\|\bfx_i-\bfx_j\|$) and \emph{monotone} (its value decreases as $\|\bfx_i-\bfx_j\|$ increases). To keep the presentation specific, we make the following assumption in the remainder of the paper.
\begin{assumption*}
$k(\bfx_i, \bfx_j) := \eta \exp\prl{-\gamma\|\bfx_i - \bfx_j\|^2}$
\end{assumption*}
The kernel parameters $\eta, \gamma \in \mathbb{R}_{>0}$ can be optimized offline via automatic relevance determination~\cite{neal2012bayesian} using training data from known occupancy maps. 
The assumption implies that the value of $F(\bfx)$ is not affected significantly by support vectors far from $\bfx$. We introduce an $R^*$-tree data structure constructed from the sets of positive and negative support vectors $\Lambda^+$, $\Lambda^-$ to allow efficient $K$ nearest-neighbor lookup. The $K$ nearest support vectors, consisting of $K^+$ and $K^-$ positive and negative support vectors, are used to approximate the score $F(\bfx)$ (Lines 1-3 in Alg.~\ref{alg:fastron_model}).

\section{Collision Checking with Kernel-based Maps}
\label{subsec:collison_checking_with_fastron_map}

A map representation is useful for navigation only if it allows checking a potential robot trajectory $\bfs(t)$ over time $t$ for collisions.
We derive conditions for complete (without sampling) collision-checking of continuous-time trajectories $\bfs(t)$ in our sparse kernel-based occupancy map representation. Checking that a curve $\bfs(t)$ is collision-free is equivalent to verifying that $F(\bfs(t)) < 0$, $\forall t \geq 0$. It is not possible to express this condition for $t$ explicitly due to the nonlinearity of $F$. Instead, in Prop.~\ref{prop:score_bounds}, we show that an accurate upper bound $U(\bfs(t))$ of the score $F(\bfs(t))$ exists and can be used~to evaluate the condition $U(\bfs(t)) < 0$ explicitly in terms of $t$.
\begin{proposition}
\label{prop:score_bounds}
For any $(\bfx_j^-, \alpha_j^-)\in \Lambda^-$, the score $F(\bfx)$ is bounded above by $U(\bfx)=k(\bfx, \bfx_*^+) \sum_{i = 1} ^ {M^+} \alpha^+_i \!-\! k(\bfx, \bfx_j^-)\alpha^-_j$ where $\bfx_{*}^{+}$ is the closest positive support vector to $\bfx$.
\end{proposition}
\begin{proof}
\vspace{-1mm}
The proposition holds because $k(\bfx, \bfx_i^+) \leq k(\bfx, \bfx_*^+)$, $\forall \bfx_i^+$ and $\sum_{j = 1}^{M^-} \alpha_j^- k(\bfx, \bfx_j^-) \geq \alpha_j^- k(\bfx, \bfx_j^-)$, $\forall \bfx_j^-$.
\end{proof}
Fig.~\ref{fig:boundary_example}, \ref{fig:upperbound_line1}, \ref{fig:upperbound_line2} illustrate the exact decision boundary $F(\bfx) = 0$ and the accuracy of the upper bound $U(\bfs(t))$ along two lines $\bfs(t)$  in C-space. The upper bound $U(\bfs(t))$  is loose in the occupied space but remains close to the score $F(\bfs(t))$ in the free space since the RBF kernel $k(\bfx,\bfx')$ is negligible away from the obstacle's boundary. As a result, the boundary $U(\bfx) = 0$ remains close to the true decision boundary. The upper bound provides a conservative but fairly accurate ``inflated boundary'', allowing efficient collision checking for line segments and polynomial curves as shown next.
\subsection{Collision Checking for Line Segments}
\label{subsubsec:collision_check_line}
Suppose that the robot's path is described by a ray $\bfs(t) = \bfs_0 + t\bfv$, $t\geq 0$ such that $\bfs_0$ is obstacle-free, i.e., $U(\bfs_0) < 0$, and $\bfv$ is constant. To check if $\bfs(t)$ collides with the inflated boundary, we find the first time $t_u$ such that $U(\bfs(t_u)) \geq 0$. This means that $\bfs(t)$ is collision-free for $t \in [0,t_u)$. 
\begin{proposition}
\label{prop:line_curve_defensive_checking} Consider a ray $\bfs(t) = \bfs_0 + t\bfv$, $t\geq 0$ with $U(\bfs_0) < 0$. Let $\bfx_i^+$ and $\bfx_j^-$ be arbitrary positive and negative support vectors. Then, any point $\bfs(t)$ is free as long as:
\setlength{\abovedisplayskip}{3pt}
\setlength{\belowdisplayskip}{3pt}
\begin{equation}
\label{eq:line_curve_t_condition}
t < t_u := \min_{i \in \{1, \ldots, M^+\}} \rho (\bfs_0, \bfx_i^+, \bfx_j^-)
\end{equation}
where $\beta = \frac{1}{\gamma}\left(\log (\alpha_j^-) - \log (\sum_{i = 1} ^ {M^+} \alpha_i^+)\right)$ and \\$\scaleMathLine[1.0]{\rho(\bfs_0, \bfx_i^+, \bfx_j^-) = 
\begin{cases} 
      +\infty, & \text{if}\ \bfv^T(\bfx_i^+ - \bfx_j^-) \leq 0  \\
      \frac{\beta - \Vert \bfs_0 - \bfx_j^-\Vert ^2 - \Vert \bfs_0 + \bfx_i^+ \Vert ^2}{2\bfv^T(\bfx_j^- - \bfx_i^+)}, & \text{if}\ \bfv^T(\bfx_i^+ - \bfx_j^-) > 0 
\end{cases}}$.
\end{proposition}
\begin{proof} 
From Prop. \ref{prop:score_bounds}, a point $\bfs(t)$ is free if $U(\bfs(t)) < 0$ or 
\begin{equation}
\label{eq:line_curve_t_condition_x_star}
t < \rho(\bfs_0, \bfx_*^+, \bfx_j^-)
\end{equation}
Since $\bfx_*^+$ varies with $t$ but belongs to a finite set, $U(\bfs(t)) < 0$ if we take the minimum of $\rho (\bfs_0, \bfx_i^+, \bfx_j^-)$ over all $ \bfx_i^+$.\qedhere
\end{proof}

Prop. \ref{prop:score_bounds} and \ref{prop:line_curve_defensive_checking} hold for any negative support vector $\bfx_j^-$. Since $\bfx_j^-$ belongs to a finite set, we can take the best bound on $t$ over the set of negative support vectors.
\begin{corollary}
\label{collorary:line_curve_tigher_bound} Consider a ray $\bfs(t) = \bfs_0 + t\bfv$, $t\geq 0$ with $U(\bfs_0) < 0$. Let $\bfx_i^+$ and $\bfx_j^-$ be arbitrary positive and negative support vectors, respectively. A point $\bfs(t)$ is free as long as:
\begin{equation}
\label{eq:line_curve_t_condition_tigher_bound}
t <  t_u^* := \min_{i \in \{1, \ldots, M^+\}} \max_{j \in \{1, \ldots, M^-\}} \rho (\bfs_0, \bfx_i^+, \bfx_j^-).
\end{equation}
\end{corollary}
The computational complexities of calculating $t_u$ and $t^*_u$ are $O(M)$ and $O(M^2)$, respectively, where $M= M^+ + M^-$. 
Note that since often the robot's movement is limited to the neighborhood of its current position, $t_u$ can reasonably approximate $t^*_u$ if $\bfx_j^-$ is chosen as the negative support vector, closest to the current location $\bfs_0$.
Calculation of $t_u$ in Eq.~\eqref{eq:line_curve_t_condition} is efficient in the sense that it has the same complexity as checking a point for collision ($O(M)$), yet it can evaluate the collision status for an entire line segment for $t \in [0, t_u)$ \NEW{without sampling}.

Collision checking becomes slower when the number of support vectors $M$ increases. We improve this further by using $K^+$ and $K^-$ nearest positive and negative support vectors instead of $M^+$ and $M^-$, respectively. Assuming $K^+$ and $K^-$ are constant, the computational complexities of calculating $t_u$ and $t^*_u$ reduce to $O(\log M)$ which is the complexity of an $R^*$-tree lookup.

In path planning, one often performs collision checking for a line segment $(\bfs_A, \bfs_B)$. All points on the segment can be expressed as $\bfs(t_A) = \bfs_A + t_A\bfv_A$, $\bfv_A = \bfs_B - \bfs_A$, $0 \leq t_A \leq 1$. Using the upper bound $t_{uA}$ on $t_A$ provided by Eq.~\eqref{eq:line_curve_t_condition} or Eq.~\eqref{eq:line_curve_t_condition_tigher_bound}, we find the free region on $(\bfs_A, \bfs_B)$ starting from $\bfs_A$. Similarly, we calculate $t_{uB}$ which specifies the free region from $\bfs_B$. If $t_{uA} + t_{uB} > 1$, the entire line segment is free; otherwise the segment is considered colliding. The proposed approach is summarized in Alg.~\ref{alg:collision_checking_line} and illustrated in Fig.~\ref{fig:collision_checking}.
\begin{algorithm}[h]
\caption{Line segment collision check}
\label{alg:collision_checking_line}
\footnotesize
	\begin{algorithmic}[1]	  
		\Require Line segment $(\bfs_A, \bfs_B)$; support vectors $\Lambda^+ = \{(\bfx_i^+, \alpha_i^+)\}$ and $\Lambda^- = \{(\bfx_j^-, \alpha_j^-)\}$
		\State $\bfv_A = \bfs_B - \bfs_A$, $\bfv_B = \bfs_A - \bfs_B$
		\State Calculate $t_{uA}$ and $t_{uB}$ using Eq.~\eqref{eq:line_curve_t_condition} or Eq.~\eqref{eq:line_curve_t_condition_tigher_bound}
    \If{$t_{uA} + t_{uB} > 1$} \Return True (Free)
		\Else \;\Return False (Colliding)
		\EndIf
	\end{algorithmic}
\end{algorithm}
\begin{figure}[t]
\vspace{-3mm}
\centering
\includegraphics[width=0.5\linewidth]{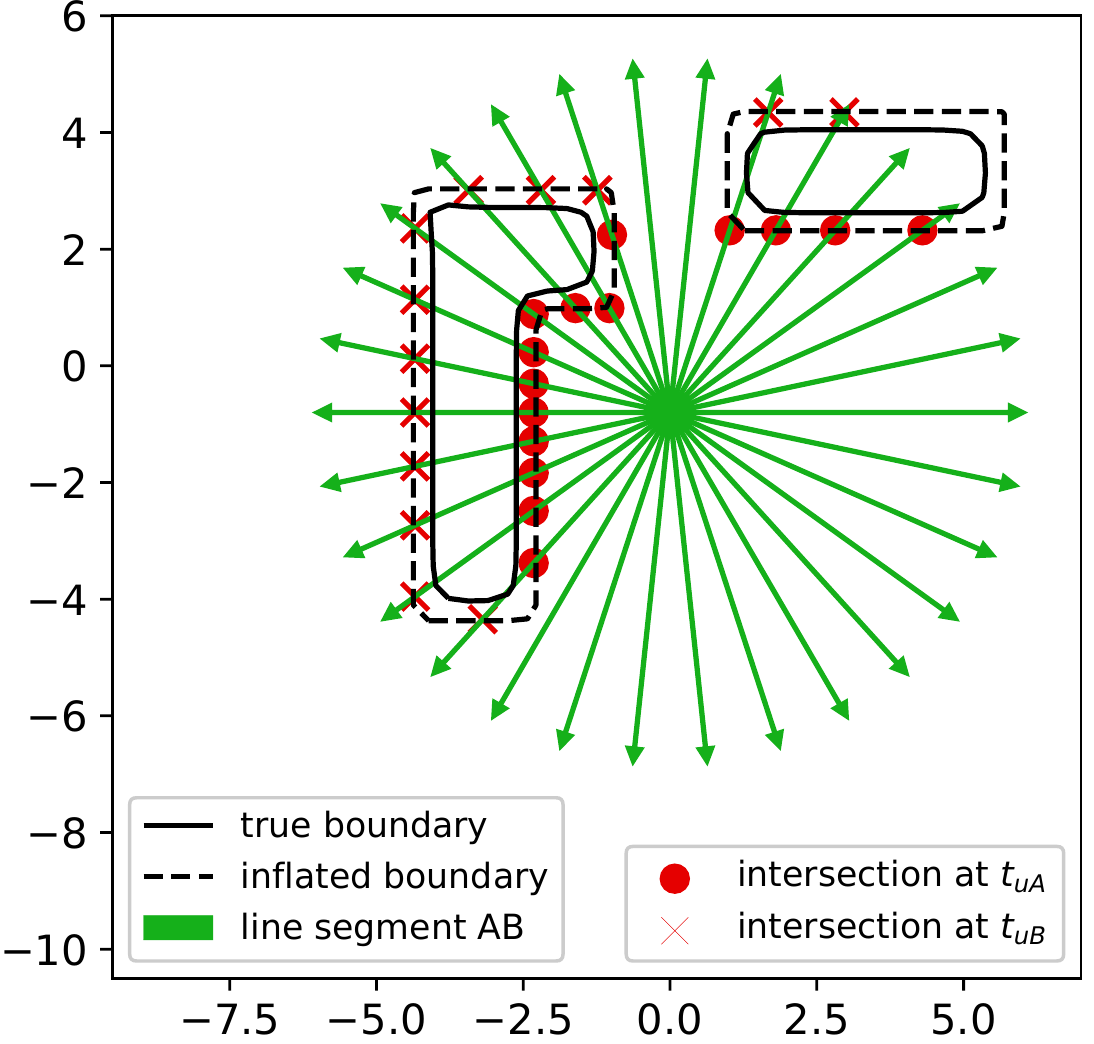}\hfill%
\includegraphics[width=0.49\linewidth]{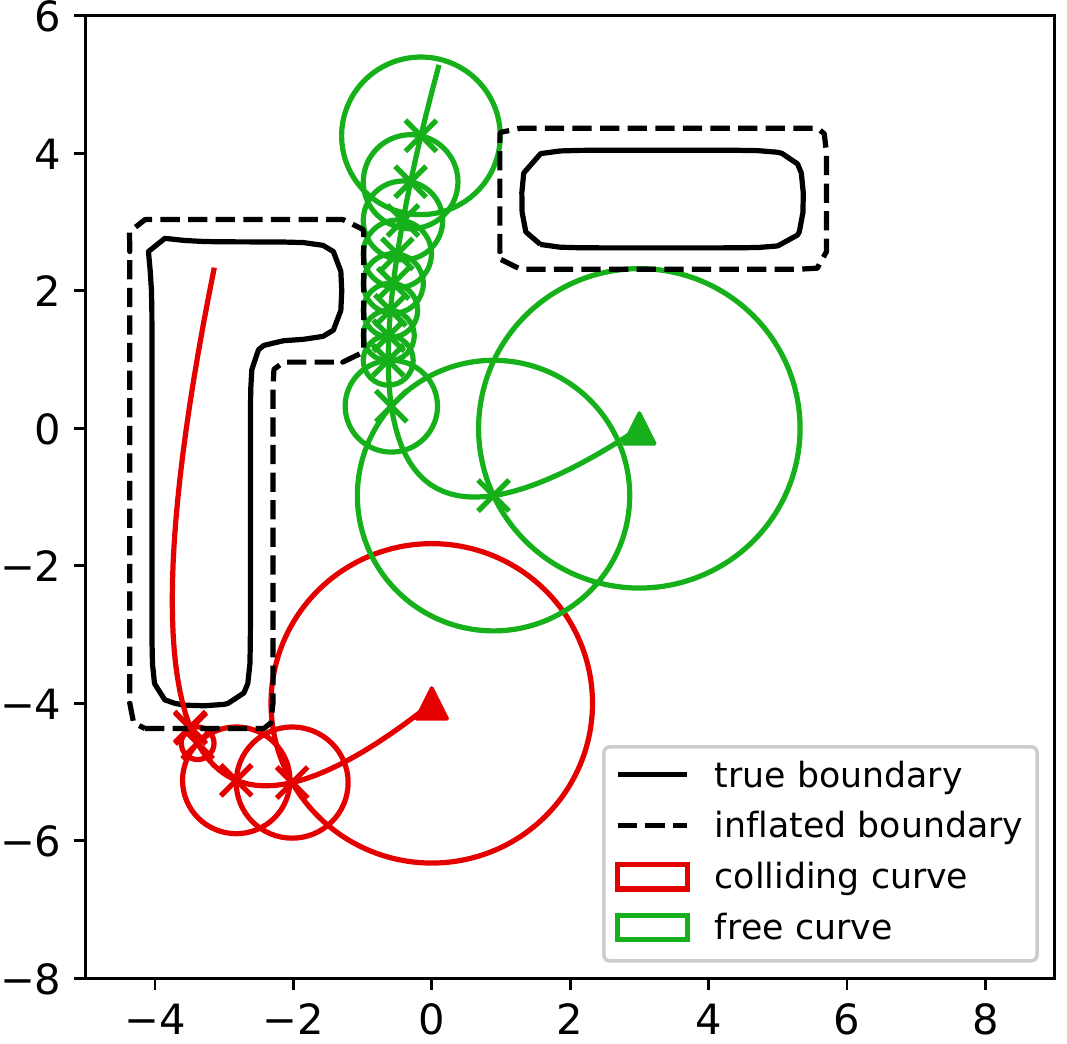}
\caption{Collision checking for line segments (left), with bounds $t_{uA}$ and $t_{uB}$ obtained from Eq.~\eqref{eq:line_curve_t_condition_tigher_bound}, and for second-order polynomial curves (right) using Euclidean balls.}
\label{fig:collision_checking}
\vspace{-3mm}
\end{figure}
\begin{algorithm}[h!]
\caption{Polynomial curve collision check}
\label{alg:collision_checking_curve}
  \footnotesize
	\begin{algorithmic}
		\Require Polynomial curve $\bfs(t)$, $t \in [0,t_f]$; threshold $\varepsilon$; support vectors $\Lambda^+ = \{(\bfx_i^+, \alpha_i^+)\}$ and $\Lambda^- = \{(\bfx_j^-, \alpha_j^-)\}$, $k = 0$, $t_0 = 0$
		\While{True}
			\State Calculate $r_{k}$ using Eq. \eqref{eq:line_curve_radius} or Eq. \eqref{eq:line_curve_radius_star}.
			\If{$r_k < \varepsilon$} \Return False (Colliding)
			\EndIf
			\State Solve $\Vert \bfs(t) - \bfs(t_k) \Vert = r_k$ for $t_{k+1} \geq t_{k}$
			\If{$t_{k+1} \geq t_f$}
				\Return True (Free)
			\EndIf
		\EndWhile
	\end{algorithmic}
\end{algorithm}

\subsection{Collision Checking for Polynomial Curves}
\label{subsubsec:collision_check_poly_curve}
In the previous section, $\bfv$ was a constant velocity representing the direction of motion of the robot. In general, the velocity might be time varying leading to trajectories described by polynomial curves~\cite{liu2017search}. We extend the collision checking algorithm by finding an Euclidean ball $\mathcal{B}(\bfs_0, r)$ around $\bfs_0$ whose interior is free of obstacles. 
\begin{corollary}
\label{corollary:free_ball} Let $\bfs_0 \in \mathcal{C}$ be such that $U(\bfs_0) <0$ and let $\bfx_i^+$ and $\bfx_j^-$ be arbitrary positive and negative support vectors. Then, every point inside the Euclidean balls $\mathcal{B}(\bfs_0, r_u) \subseteq \mathcal{B}(\bfs_0, r_u^*)$ is free for:
\begin{align}
r_u &:= \min_{i \in \{1, \ldots, M^+\}} \bar{\rho}(\bfs_0, \bfx_i^+, \bfx_j^-) \label{eq:line_curve_radius}\\
r_u^* &:= \min_{i \in \{1, \ldots, M^+\}} \max_{j \in \{1, \ldots, M^-\}} \bar{\rho}(\bfs_0, \bfx_i^+, \bfx_j^-) \label{eq:line_curve_radius_star}
\end{align}
where $\bar{\rho}(\bfs_0, \bfx_i^+, \bfx_j^-) = \frac{\beta - \|\bfs_0 - \bfx_j^-\| ^2 + \|\bfs_0 - \bfx_i^+ \|^2}{2 \| \bfx_j^- - \bfx_i^+ \|}$ and $\beta = \frac{1}{\gamma} \prl{\log (\alpha_j^-) - \log (\sum_{i = 1} ^ {M^+} \alpha_i^+)}$.
\end{corollary}
\begin{proof}
Directly follows from Prop. \ref{prop:line_curve_defensive_checking} by using the Cauchy-Schwarz inequality to bound $\bfv^T (\bfx_j^- - \bfx^+_*)  \leq  \| (\bfx_j^- - \bfx^+_*)\|$ for any unit vector $\bfv$ (i.e. $\|\bfv\| = 1$).\qedhere
\vspace{-1mm}
\end{proof}
Consider a polynomial $\bfs(t) = \bfs_0 + \bfa_1 t + \bfa_2 t^2 + \ldots + {\bfa}_d t^d$, $t \in [0, t_f]$  from $\bfs_0$ to $\bfs_f := \bfs(t_f)$. Collorary~\ref{corollary:free_ball} shows that all points inside $\mathcal{B}(\bfs_0, r)$ are free for $r = r_u$ or $r^*_u$. If we can find the smallest positive $t_1$ such that $\|\bfs(t_1) - \bfs_0 \| = r$, then all points on the curve $\bfs(t)$ for $t \in [0, t_1)$ are free. This is equivalent to finding the smallest non-negative root of a $2d$-order polynomial. Note that, if $d \leq 2$, there is a closed-form solution for $t_{1}$. For higher order polynomials, one can use a root-finding algorithm to obtain $t_1$ numerically. We perform collision checking by iteratively covering the curve by Euclidean balls. If the radius of any ball is smaller than a threshold $\varepsilon$, the curve is considered colliding. Otherwise, the curve is considered free. The collision checking process for $d$-order polynomial curves is shown in Alg.~\ref{alg:collision_checking_curve} and Fig.~\ref{fig:collision_checking}. 
\section{Autonomous Navigation}
\label{sec:auto_nav}
Finally, we present a complete online mapping and navigation approach that solves Problem~\ref{problem_formulation_unknown_env}. Given the kernel-based map $\hat{m}_{t_k}$ proposed in Sec.~\ref{subsec:ogm_with_fastron}, a motion planning algorithm such as $A^*$~\cite{Russell_AI_Modern_Approach} may be used with our collision-checking algorithms to generate a path that solves the autonomous navigation problem. The robot follows the path for some time and updates the map estimate $\hat{m}_{t_{k+1}}$ with new observations. Using the updated map, the robot re-plans the path and follows the new path instead. This process is repeated until the goal is reached or a time limit is exceeded (Alg.~\ref{alg:auto_nav_fastron_ogm}).

We consider robots with two different motion models. In simulation, we use a first-order fully actuated robot, $\dot{\bfs} = \bfv$, with piecewise-constant velocity $\bfv(t) \equiv \bfv_k \in \calV$ for $t \in [t_k,\;t_{k+1})$, leading to piecewise-linear trajectories:
\begin{equation}
\label{eq:fas_trajectory}
\bfs(t) = \bfs_k + (t-t_k)\bfv_k, \qquad t \in [t_k,\;t_{k+1}),
\end{equation}
where $\bfs_k := \bfs(t_k)$. In real experiments, we consider a ground wheeled Ackermann-drive robot:
\begin{equation}
\label{eq:car_dynamics}
\dot{\bfs} = v \begin{bmatrix} \cos(\theta)\\\sin(\theta) \end{bmatrix},\qquad \dot{\theta} = \frac{v}{\ell}\tan \phi,
\end{equation}
where $\bfs \in \mathbb{R}^2$ is the position, $\theta \in \mathbb{R}$ is the orientation, $v \in \mathbb{R}$ is the linear velocity, $\phi \in \mathbb{R}$ is the steering angle, and $\ell$ is the distance between the front and back wheels. The nonlinear car dynamics can be transformed into a 2nd-order fully actuated system $\ddot{\bfs} = \bfa$ via feedback linearization~\cite{de2000stabilization, franch2009control}.
Using piecewise-constant acceleration $\bfa(t) \equiv \bfa_k \in \calA$ for $t \in [t_k,\;t_{k+1})$ leads to piecewise-2nd-order-polynomial trajectories:
\begin{equation}
\label{eq:car_trajectory}
\bfs(t) = \bfs_k + (t-t_k) v_k\begin{bmatrix} \cos(\theta_k)\\\sin(\theta_k) \end{bmatrix} + \frac{(t-t_k)^2}{2}\bfa_k,
\end{equation}
where $\bfs_k := \bfs(t_k)$, $\theta_k := \theta(t_k)$, $v_k := v(t_k)$. To apply $A^*$ to these models, we restrict the input sets $\calV$ and $\calA$ to be finite.
\section{Experimental Results}
\label{sec:experimental_results}
This section presents an evaluation of Alg.~\ref{alg:auto_nav_fastron_ogm} using a fully actuated robot~\eqref{eq:fas_trajectory} in simulation and a car-like robot (Fig.~\ref{fig:camera_images}) with Ackermann-drive dynamics~\eqref{eq:car_dynamics} in real experiments. We use an RBF kernel with parameters $\eta = 1$, $\gamma = 2.5$ and an $R^*$-tree approximation of the score $F(\bfx)$ with $K^++K^- = 200$ nearest support vectors around the robot location $\bfs_k$ for map updating. The online training data (Sec.~\ref{subsec:ogm_with_fastron}) were generated from a grid with resolution $0.25m$. Timing results are reported from an Intel i7 2.2 GHz CPU with 16GB RAM.
\subsection{Simulations}
\label{subsubsec:map_compare}
\setlength{\abovecaptionskip}{3pt}
\setlength{\belowcaptionskip}{3pt}
\begin{figure*}[t]
\centering
\includegraphics[width=0.33\textwidth]{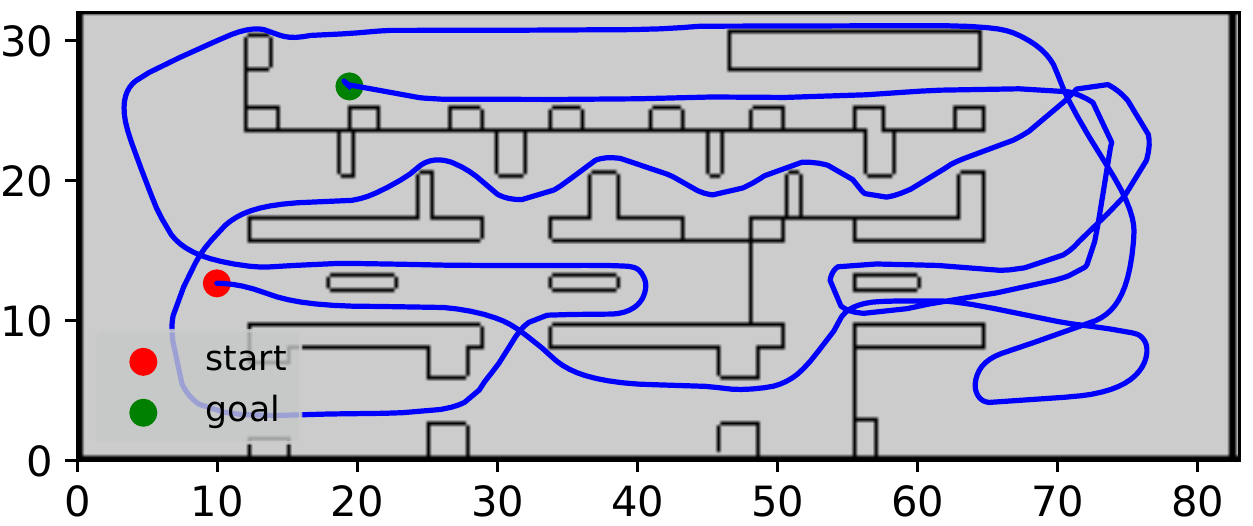}
\includegraphics[width=0.33\textwidth]{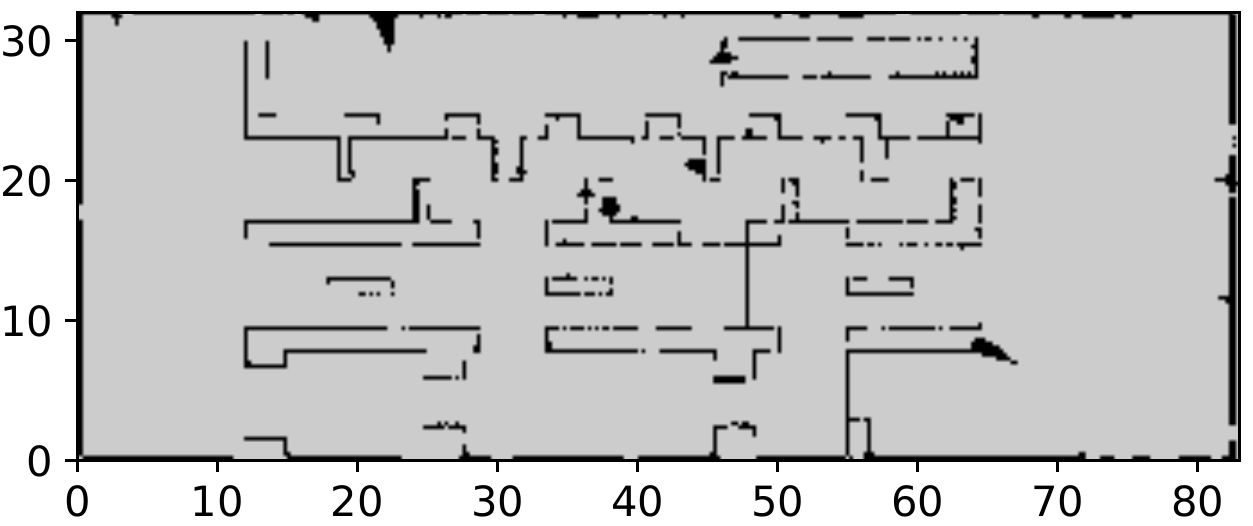}\includegraphics[width=0.33\textwidth]{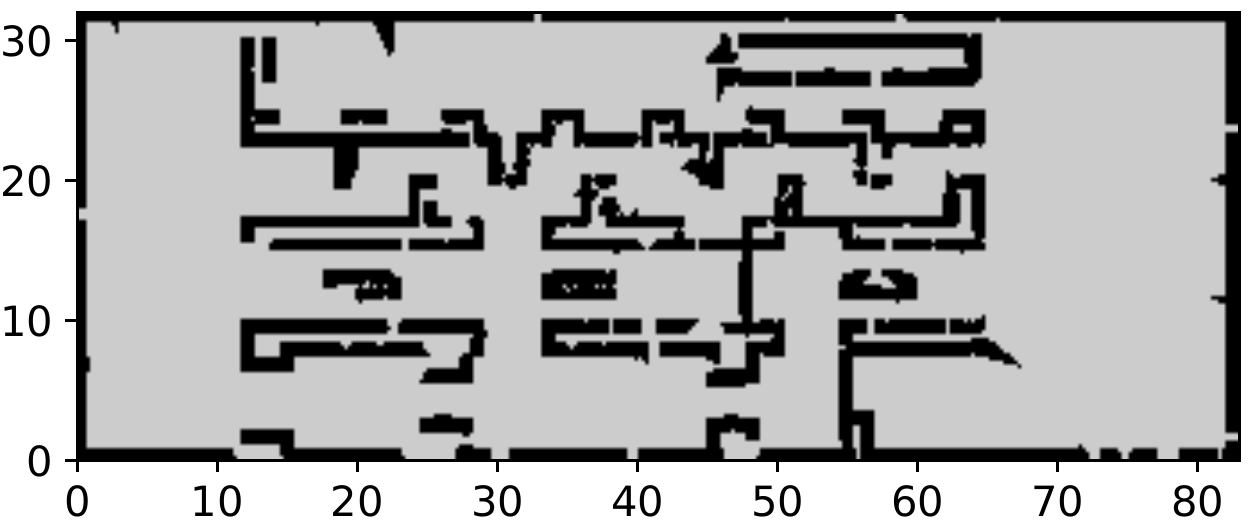}
\caption{Ground truth map and robot trajectory (left) used to generate simulated lidar scans, the occupancy map generated from our sparse kernel-based representation (middle) and the inflated map from the upper bound $U(\bfx)$ proposed in Prop.~\ref{prop:score_bounds} (right).}
\label{fig:map_comparison}
\vspace{-4mm}
\end{figure*}
\begin{figure*}[t]
\centering
\includegraphics[width=0.241\linewidth]{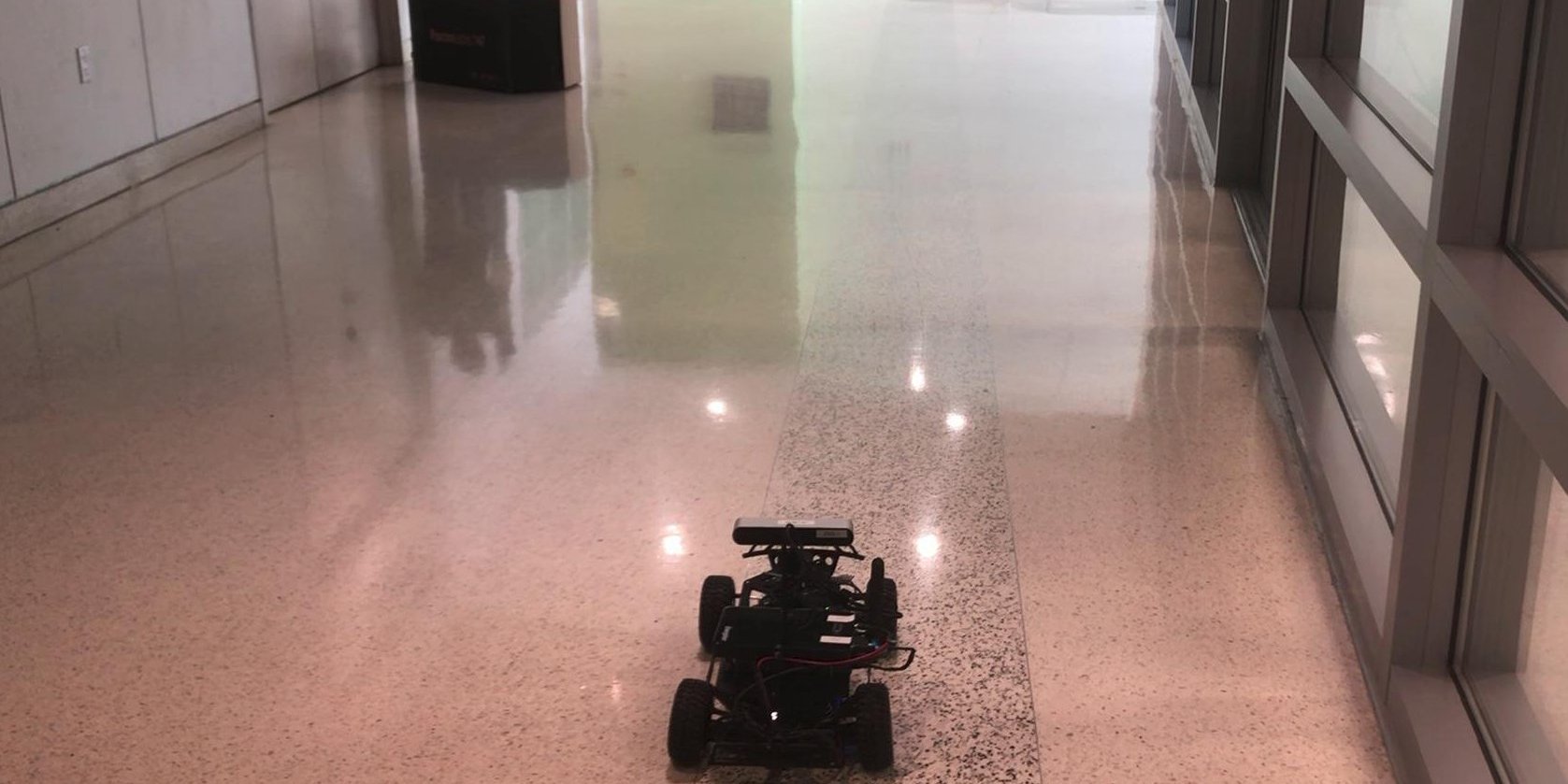}        \includegraphics[width=0.241\linewidth]{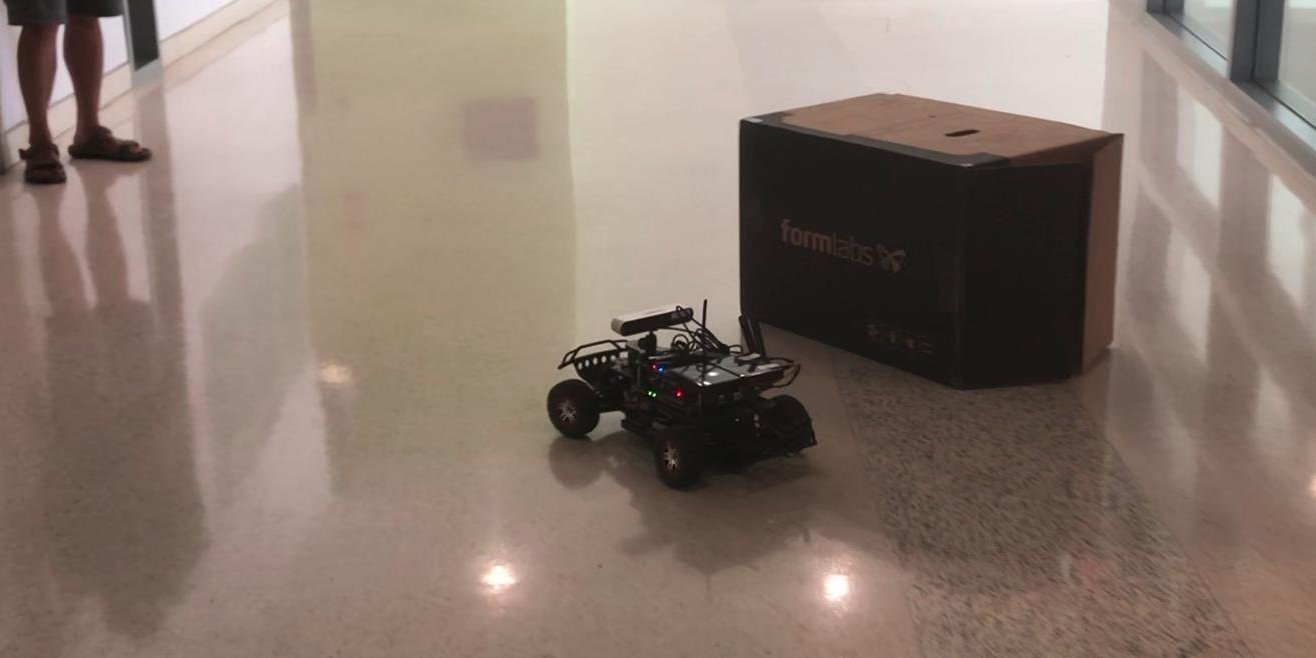}        \includegraphics[width=0.241\linewidth]{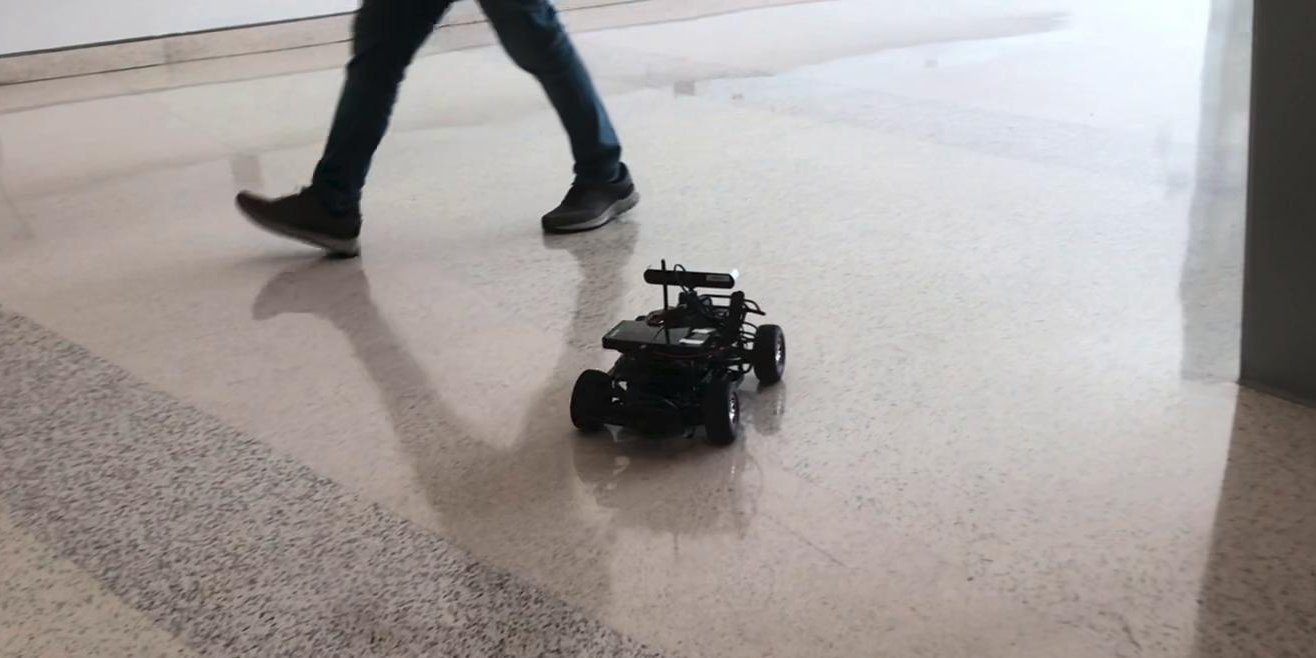}        \includegraphics[width=0.241\linewidth]{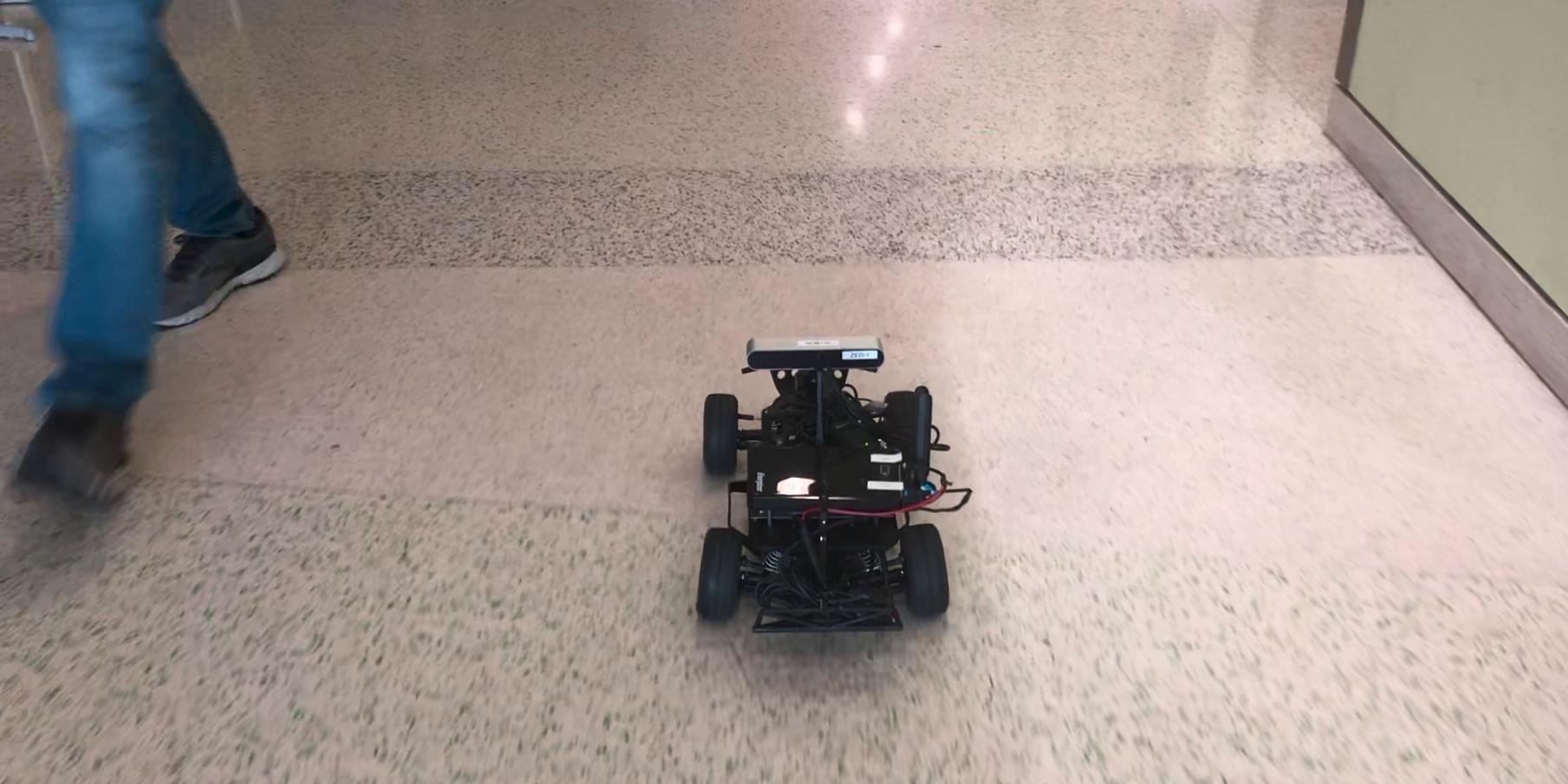}
\caption{Third person views of the autonomous Racecar robot navigating in an unknown environment with moving obstacles.}
\label{fig:camera_images}
\vspace{-4mm}
\end{figure*}
\begin{figure*}[t]
\centering
\includegraphics[width=0.325\linewidth]{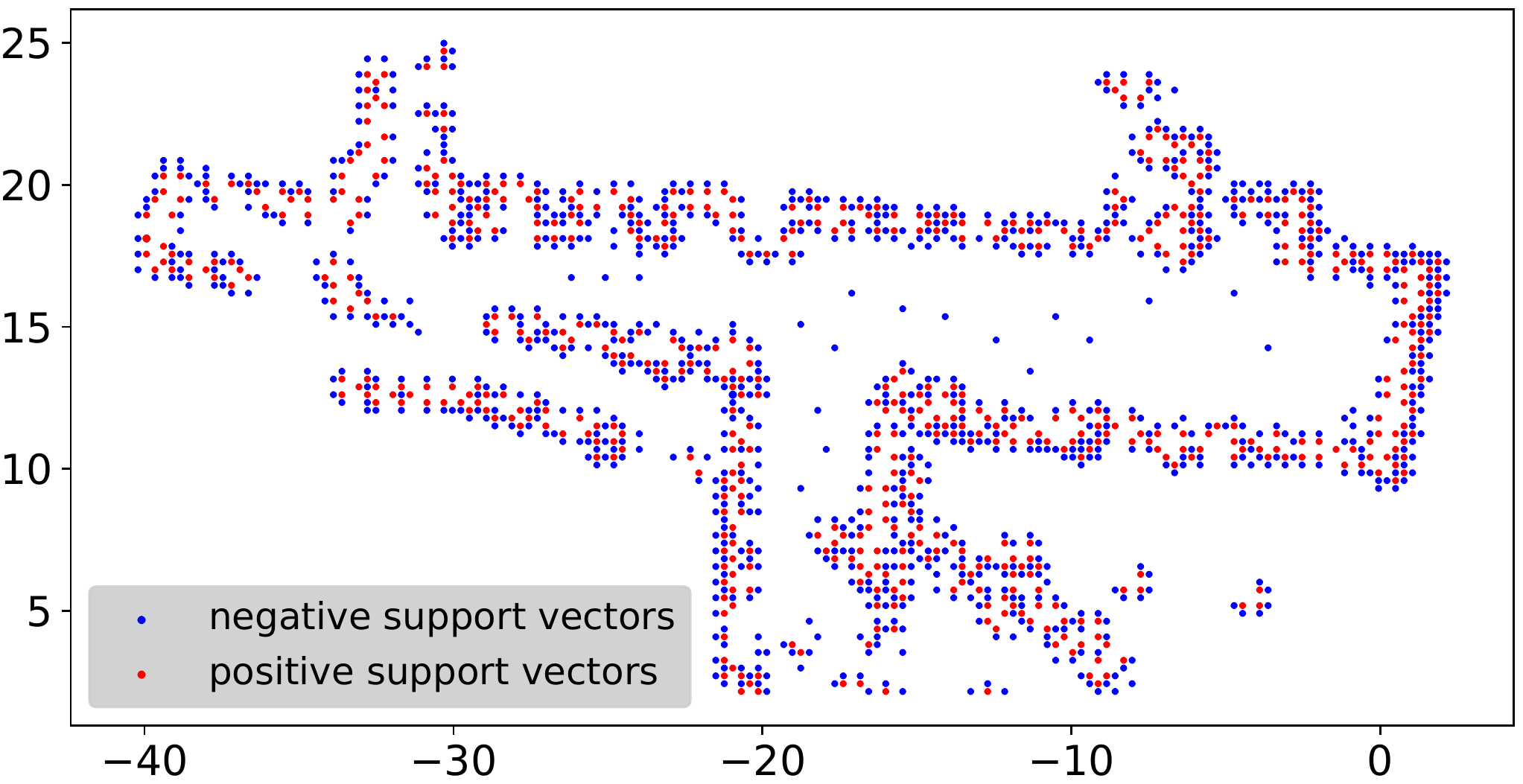}
\includegraphics[width=0.325\linewidth]{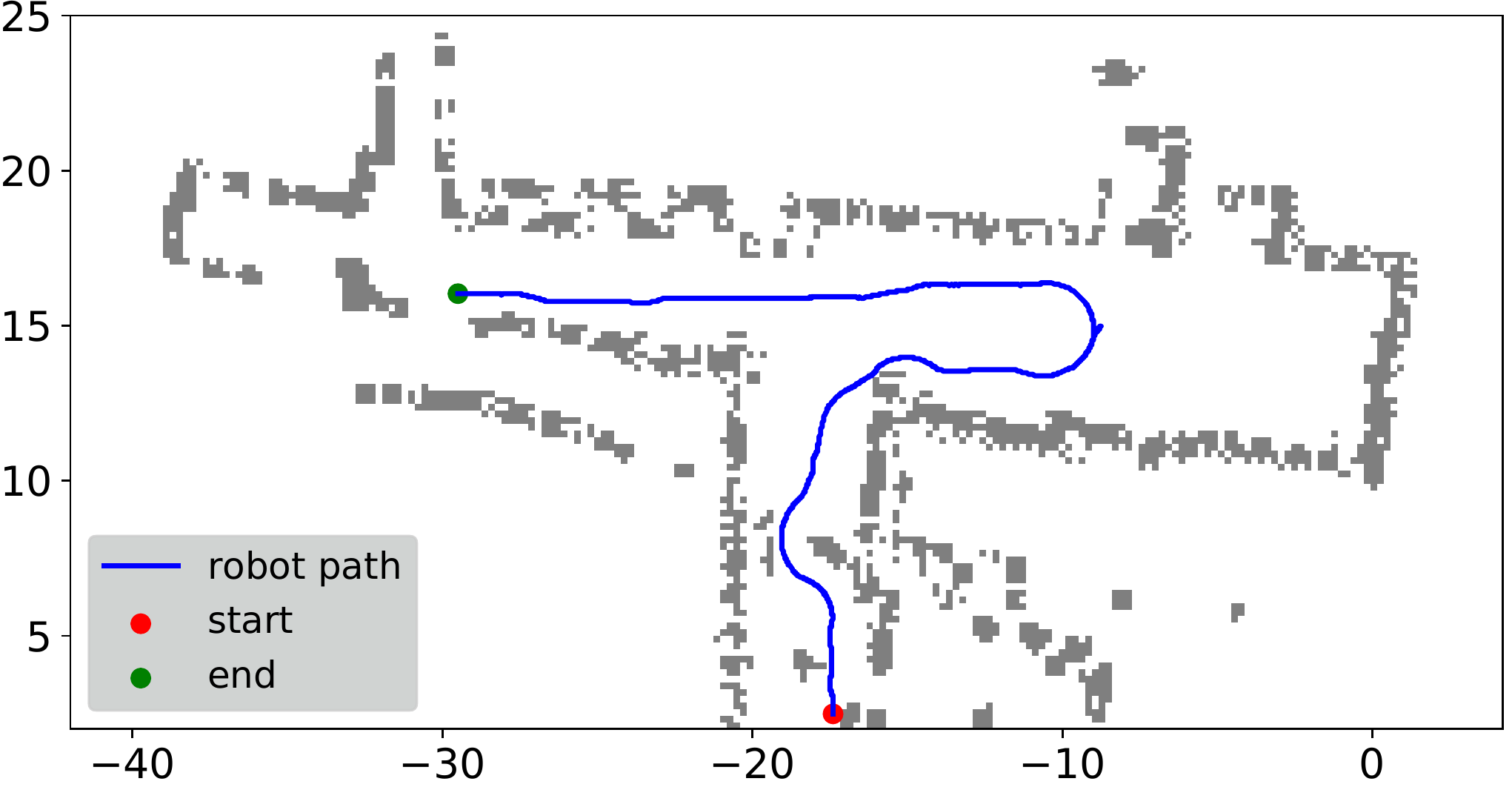}
\includegraphics[width=0.325\linewidth]{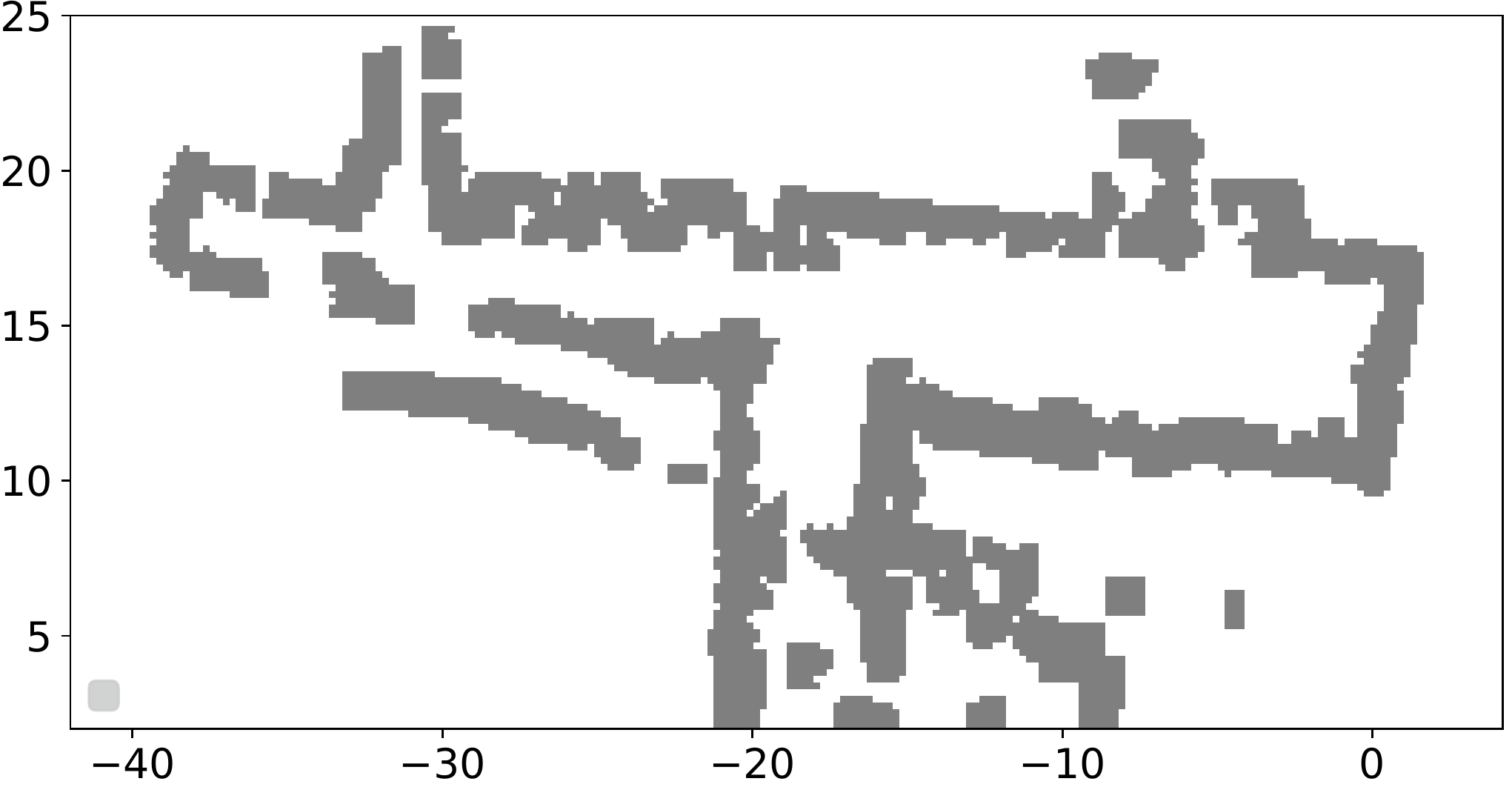}
\caption{Final $1762$ support vectors (left), kernel-based  map (middle), and inflated map (right) obtained from the real experiments.}
\label{fig:realcar_final_map}
\vspace{-4mm}
\end{figure*}
\begin{figure*}[t]
\centering
\includegraphics[width=0.33\textwidth]{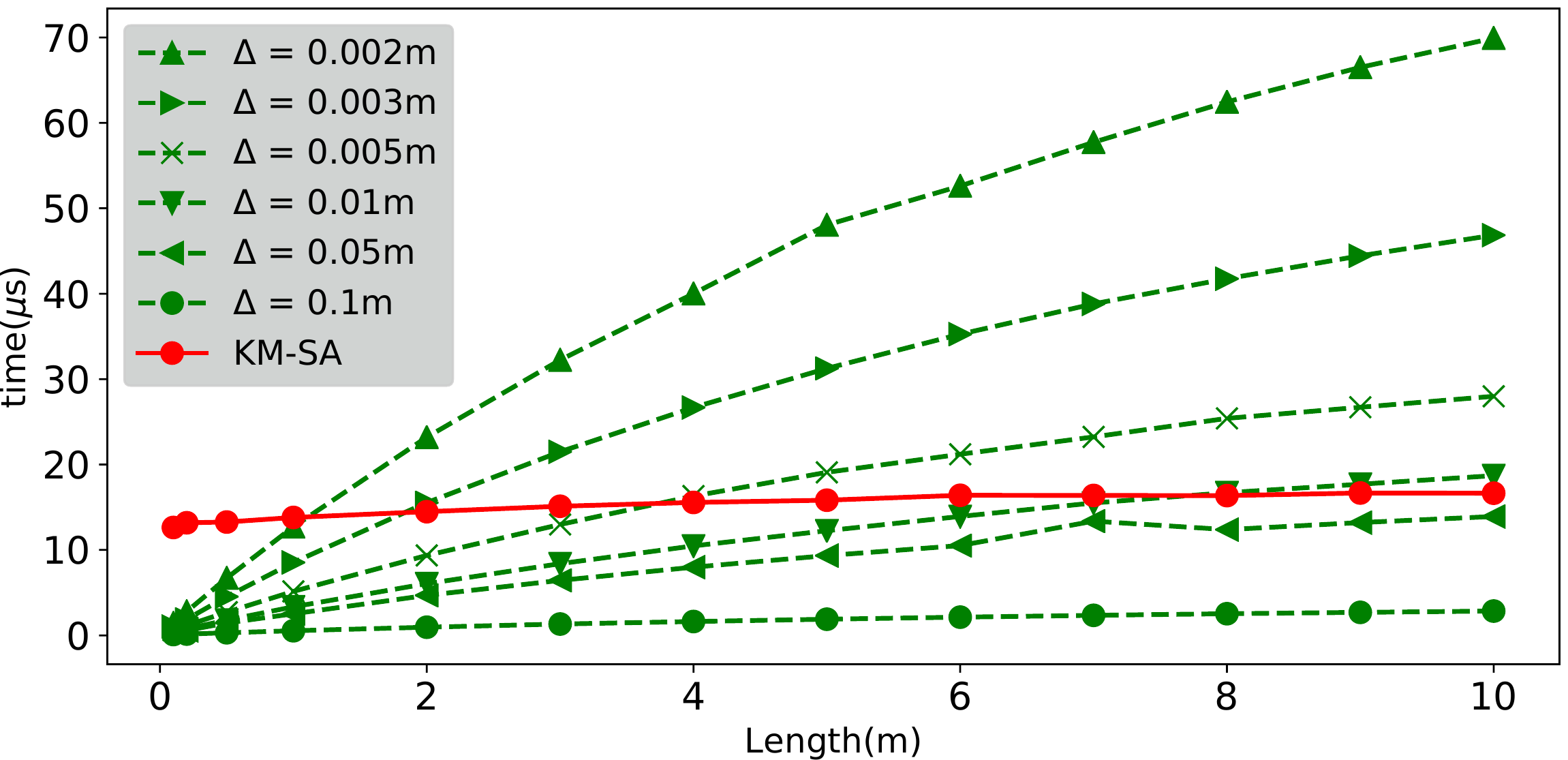}
\includegraphics[width=0.33\linewidth]{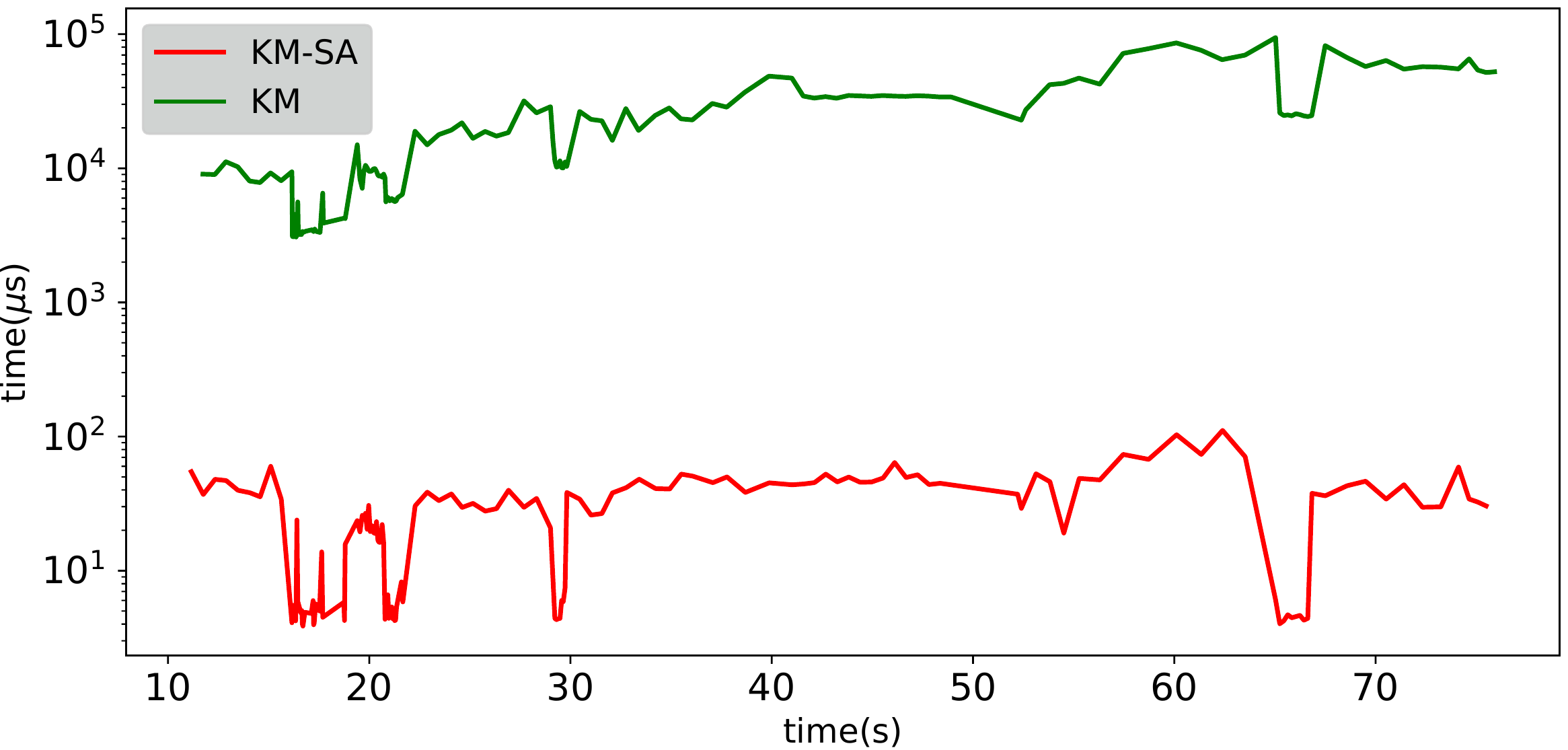}%
\includegraphics[width=0.33\linewidth]{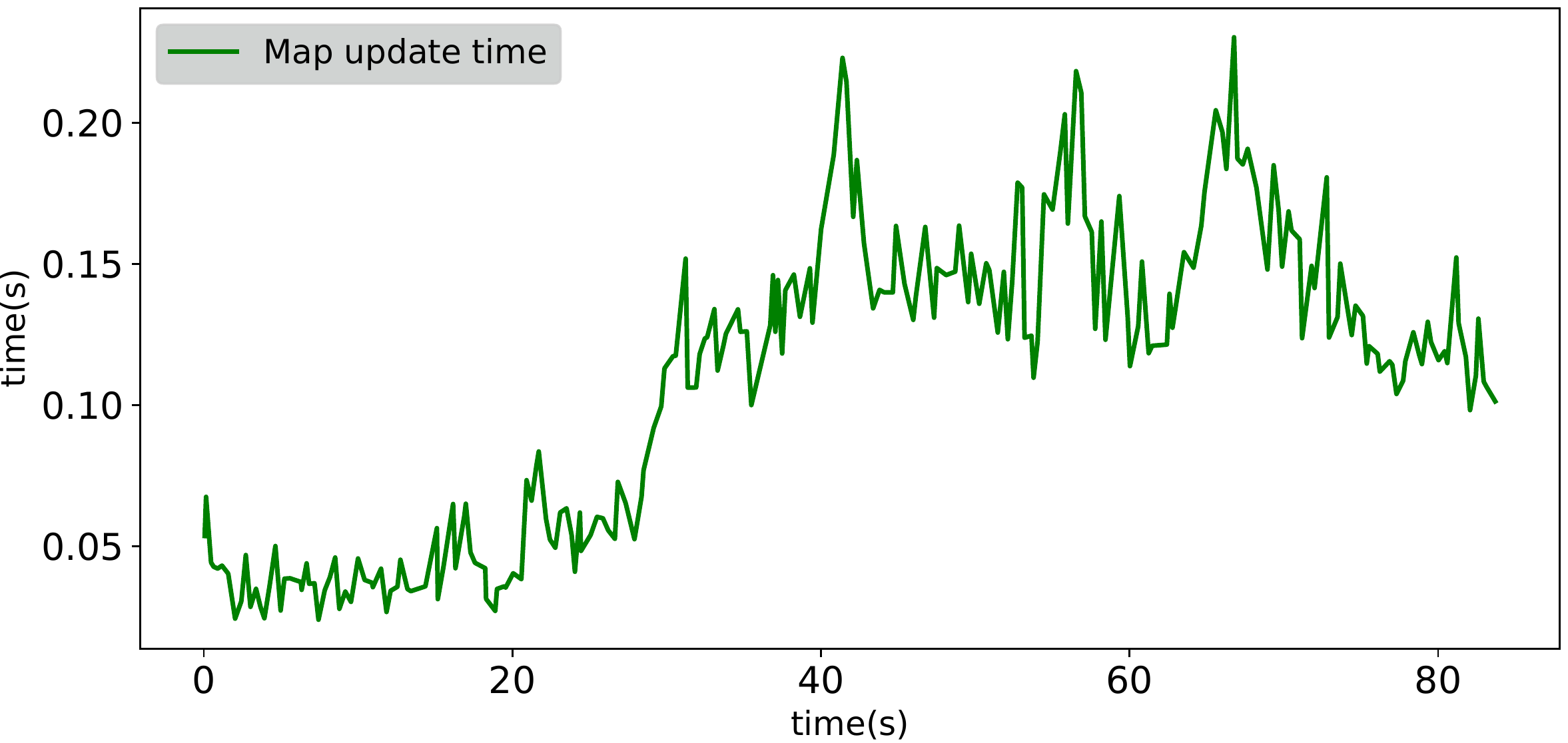}
\caption{Checking line segments in simulation (left): comparison between our method with score approximation (KM-SA) and a sampling-based method with different sampling resolution $\Delta$. Planning time per motion primitive in the real experiment (middle): comparison between our method with (KM) and without score approximation (KM-SA). Map update time for the real experiment (right).}
\label{fig:veccount_astar_map_update_time}
\vspace{-7mm}
\end{figure*}
\begin{algorithm}[h]
\caption{Autonomous Mapping and Navigation with a Sparse Kernel-based Occupancy Map}
\label{alg:auto_nav_fastron_ogm}
  \footnotesize
	\begin{algorithmic}[1]
		\Require Initial position $\bfs_0 \in \mathcal{C}_{free}$; goal region $\mathcal{C}_{goal}$; time limit $T$; prior support vectors $\Lambda_0^+$ and $\Lambda_0^-$, $t = 0$
		\While {$\bfs_t \notin \mathcal{C}_{goal}$ and $t < T$}
			\State $\bfz_t\gets$ New Depth Sensor Observation
			\State $\bf\mathcal{D}\leftarrow $ Training Data Generation$(\bfz_t, \Lambda_t^+, \Lambda_t^-)$ \Comment{Sec. \ref{subsec:ogm_with_fastron}}
			\State $\Lambda_{t+1}^+, \Lambda_{t+1}^- \leftarrow $ Incremental Fastron$(\Lambda_t^+, \Lambda_t^-, \mathcal{D}, \bfs_t)$ \Comment{Alg.~\ref{alg:fastron_model}}
			\State Path Planning$(\Lambda_{t+1}^+, \Lambda_{t+1}^-, \bfs_t, \mathcal{C}_{\text{goal}})$ \Comment{Replan via $A^*$ (\NEW{Alg.~\ref{alg:collision_checking_line}\&\ref{alg:collision_checking_curve}})}
			\State $\bfs_{t+1} = f(\bfs_t,\bfa_t)$ \Comment{Move to the first position along the path}
		\EndWhile
	\end{algorithmic}
\end{algorithm}
The accuracy and memory consumption of our sparse kernel-based map was compared with OctoMap~\cite{octomap} in a warehouse environment shown in Fig.~\ref{fig:map_comparison}. As the ground-truth map represents the work space instead of C-space, a point robot ($r=0$) was used for an accurate comparison. Lidar scan measurements were simulated along a robot trajectory shown in Fig.~\ref{fig:map_comparison} and used to build our map and OctoMap simultaneously. OctoMap's resolution was also set to $0.25m$ to match that of grid used to sample our training data from. Furthermore, since our map does not provide occupancy probability, OctoMap's binary map was used as the baseline.

Table~\ref{table:accuracy} compares the accuracy and the memory consumption of OctoMap's binary map versus our kernel-based map and inflated map (using the upper bound in Prop.~\ref{prop:score_bounds}) with and without score approximation (Sec.~\ref{subsec:ogm_with_fastron}). The kernel-based and inflated maps (with score approximation) are shown in Fig. \ref{fig:map_comparison}. The kernel-based maps and OctoMap's binary map lead to similar accuracy of $\sim 96-98\%$ compared to the ground truth map. OctoMap required a compressed octree with $12372$ non-leaf nodes with 2 bytes per node, leading to a memory requirement of $\sim24.7 kB$. As the memory consumption depends on the computer architecture and how the information on the support vectors is compressed, we provide only a rough estimate to show that our map's memory requirements are at least comparable to those of OctoMap. We stored an integer representing a support vector's location on the underlying grid and a float representing its weight. This requires 8 bytes on a 32-bit architecture per support vector. Our maps contained $1947$ and $2721$ support vectors with and without score approximation, leading to memory requirements of $15.6kB$ and $21.7kB$, respectively. 
The recall (true positive rate) reported in Table~\ref{table:accuracy} demonstrates the safety guarantee provided by our inflated map as $\sim99\%$ of the occupied cells are correctly classified.
\begin{table}[t]
\vspace{-2mm}
\centering
\begin{tabular}{ |c|c|c|c|c|c|c| } 
 \hline
 & KM &  KM-SA& IM & IM-SA & OM\\
 \hline
 Accuracy & 98.5\% & 98.5\% & 83.8\% & 83.2\% & 96.1\% \\
 \hline
 Recall & 97.4\% & 97.3\% & 99.0\% & 98.9\% & 96.8\% \\
 \hline
 Vectors/Nodes & 1947 & 2721 & 1947 & 2721 & 12372  \\
 \hline
 Storage & 15.6kB & 21.7kB & 15.6kB & 21.7kB & 24.7kB  \\ 
 \hline
\end{tabular}
\caption{Comparison among our kernel-based map (KM), KM map with score approximation (KM-SA), our inflated map (IM), IM map with score approximation (IM-SA) and OctoMap (OM)~\cite{octomap}.}
\label{table:accuracy}
\end{table}

We also compared the average collision checking time over $1000000$ random line segments using our complete method (Alg. \ref{alg:collision_checking_line} with Eq. \eqref{eq:line_curve_t_condition_tigher_bound} and $K^+ = K^- = 10$ for score approximation) and sampling-based methods with different sampling resolutions using the ground truth map. Fig.~\ref{fig:veccount_astar_map_update_time} shows that the time for sampling-based collision checking increased as the line length increased or the sampling resolution decreased. Meanwhile, our method's time was stable at $\sim 15\mu s$ suggesting its suitability for real-time applications.

\subsection{Real Robot Experiments}
\label{sec:real_experiments}
Real experiments were carried out on an $1/10$th scale Racecar robot (Fig.~\ref{fig:camera_images}) equipped with a Hokuyo UST-10LX Lidar and Nvidia TX2 computer. The robot body was modeled by a ball of radius $r = 0.25m$. Second-order polynomial motion primitives were generated with time discretization of $\tau = 1$s as described in Sec.~\ref{sec:auto_nav}. The motion cost was defined as $c(\bfs, \bfa) := (\|\bfa\|^2 + 2)\tau$ to encourage both smooth and fast motion~\cite{liu2017search}. Alg.~\ref{alg:collision_checking_curve} with Eq.~\eqref{eq:line_curve_radius_star}, $\varepsilon = 0.2$, and score approximation with $K^+ = K^- = 2$ was used for collision checking. The trajectory generated by an $A^*$ motion planner was tracked using a closed-loop controller~\cite{arslan2016exact}. The robot navigated in an unknown hallway with moving obstacles to destinations randomly chosen by a human operator. Fig.~\ref{fig:realcar_final_map} shows the support vectors, the kernel-based map, and the inflated map with score approximation for the experiment.

We observed that kernel-based mapping is susceptible to noise since the support vectors are quickly updated with newly observed data, even though it is noisy and affected by localization errors. This is caused by the kernel perceptron model not maintaining occupancy probabilities. Future work will focus on sparse generative models for occupancy maps.

The time taken by Alg.~\ref{alg:fastron_model} to update the support vectors from one lidar scan and the $A^*$ replanning time per motion primitive with and without score approximation are shown in Fig.~\ref{fig:veccount_astar_map_update_time}. Map updates implemented in Python took $0.11$s on average. To evaluate collision checking time, the $A^*$ replanning time was normalized by the number of motion primitives being checked to account for differences in planning to nearby and far goals. Without score approximation, the planning time per motion primitive was in the order of milliseconds  and increased over time as more support vectors were added. With score approximation, it was stable at $\sim 40\mu s$ illustrating the benefits of our $R^*$-tree data structure.

\section{Conclusion}
\label{sec:conclusion}
This paper proposes a sparse kernel-based mapping method for a robot navigating in an unknown environment. The method offers efficient map storage that scales with obstacle complexity rather than environment size. We developed efficient and complete collision checking for linear and polynomial trajectories in this new map representation. The experimental results show the potential of our approach to provide compressed, yet accurate, occupancy representations of large environments. The developed mapping and collision checking algorithms offer a promising avenue for safe, real-time, long-term autonomous navigation in unpredictable and rapidly changing environments. Future work will explore simultaneous localization and mapping, sparse generative models that account for occupancy probability, and active exploration and map uncertainty reduction.

{\small
\bibliographystyle{cls/IEEEtran}
\bibliography{bib/thai_ref.bib}
}


\end{document}